\documentclass{article}
\pdftrailerid{}
\PassOptionsToPackage{table}{xcolor}
\usepackage{iclr2027_conference,times}

\iclrfinalcopy

\usepackage{amsmath,amssymb,amsthm}
\usepackage{booktabs}
\usepackage{multirow}
\usepackage{array}
\usepackage{graphicx}
\usepackage{wrapfig}
\usepackage{placeins}
\usepackage{hyperref}
\usepackage{natbib}
\usepackage{colortbl}
\usepackage{relsize}
\usepackage{xspace}

\definecolor{oursrow}{RGB}{226,235,249}   
\definecolor{grouprow}{RGB}{234,237,241}  
\newcommand{\method}{CAST\xspace}
\newcommand{\tabsmall}{\relsize{-1}}
\newcommand{\secondarycolor}[1]{\textcolor{gray}{#1}}
\definecolor{baselinerow}{RGB}{246,247,248}
\definecolor{posdelta}{RGB}{28,103,80}
\definecolor{negdelta}{RGB}{150,64,35}
\newcommand{\ours}{\rowcolor{oursrow}}
\definecolor{groupband}{rgb}{1.0,0.95,0.9}
\newcommand{\grouphead}[2]{\rowcolor{groupband}\multicolumn{#1}{c}{\textit{#2}\strut}}
\definecolor{ourblue}{rgb}{0.368,0.507,0.71}
\hypersetup{colorlinks,linkcolor=ourblue,citecolor=ourblue,urlcolor=ourblue}
\newtheorem{theorem}{Theorem}
\newtheorem{proposition}{Proposition}
\newtheorem{lemma}{Lemma}
\newtheorem{corollary}{Corollary}
\newtheorem{assumption}{Assumption}

\newcommand{\Eround}{\mathbb{E}[R]}

\title{\raisebox{-0.62em}{\includegraphics[height=2.4em]{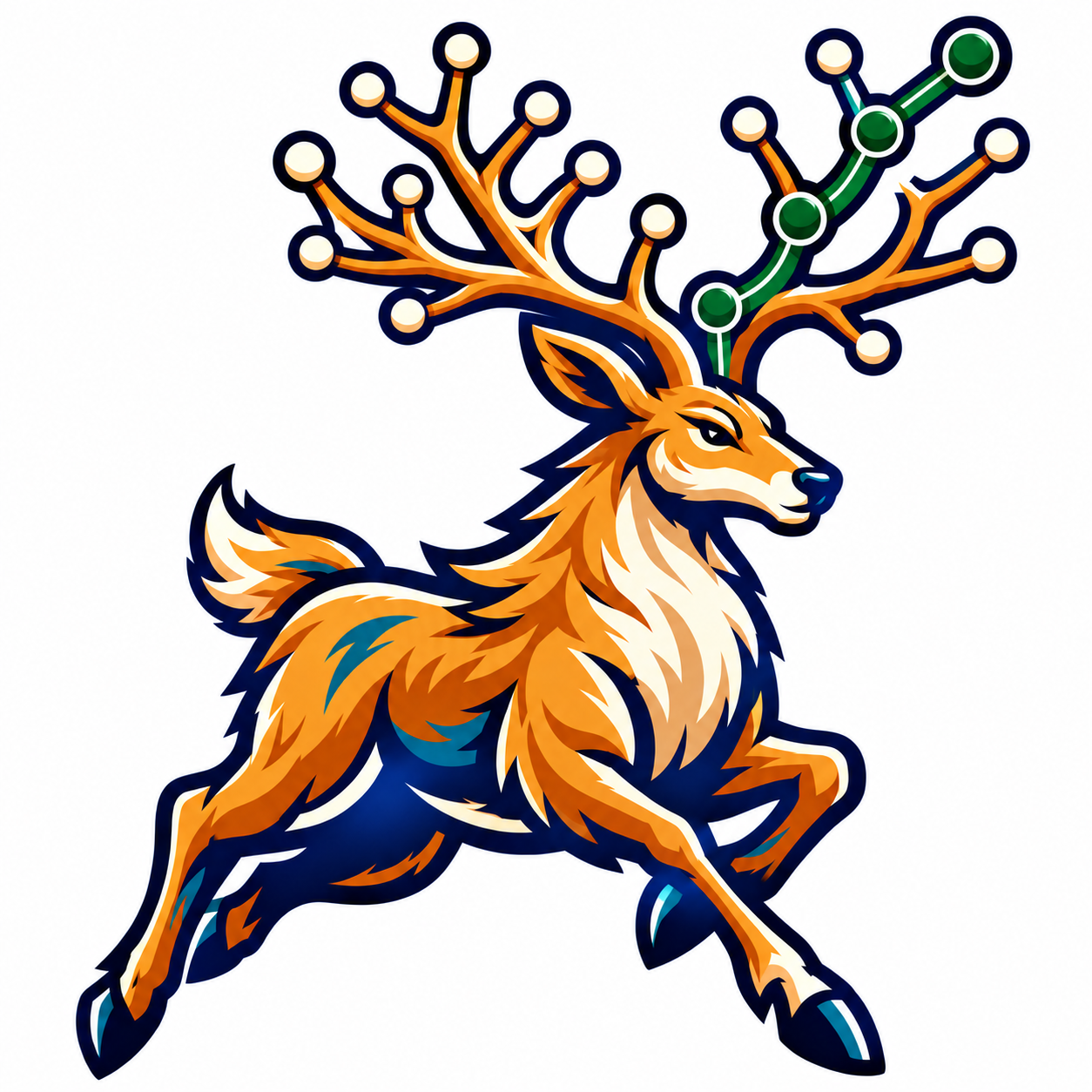}}\hspace{0.28em}CAST: Cost-Aware Speculative Trees from One-Pass Block Drafters}
\author{
\parbox[t]{2.45in}{\centering\normalfont
\textbf{Jungseob Lee} \\
Korea University \\
\texttt{omanma1928@korea.ac.kr}}
\And
\parbox[t]{2.45in}{\centering\normalfont
\textbf{Sugyeong Eo\textsuperscript{\textdagger}} \\
Yonsei University Mirae Campus \\
\texttt{s.eo@yonsei.ac.kr}}
}
\date{}

\begin{document}
\maketitle
\lhead{Preprint}
\begingroup
\renewcommand{\thefootnote}{\textdagger}
\footnotetext[1]{Corresponding author.}
\endgroup

\begin{abstract}
Speculative decoding accelerates large language model inference by drafting future tokens cheaply and verifying them with the target model in parallel. Block drafters score a whole block of future tokens in one forward pass, yet standard decoding verifies only the top-scoring chain and discards the other candidates. Because these candidates are already scored, verifying more of them adds target computation but no extra drafting. We introduce \method (Cost-Aware Speculative Trees), which packs these candidates into a tree and verifies it in a single target pass, leaving the target model, drafter weights, and decoding rule untouched. To decide how wide the tree should be, \method adds candidates while the expected gain from the next one outweighs the verification time it adds. The width therefore adapts to each deployment from a latency measurement, without sweeping over widths. We evaluate \method across five domains on three GPU generations and two model families. At its predicted width, \method is faster than the standard chain in all eight settings, by up to $43\%$. We also find that the best width depends strongly on the deployment. Where verification cost jumps at a kernel boundary, a 128-token tree is only $2\%$ faster than the standard chain, whereas the tree at the predicted width is $20\%$ faster. Furthermore, we prove that \method leaves the target output distribution unchanged under both greedy and sampled decoding. Code is available at \url{https://github.com/js-lee-AI/CAST}.
\end{abstract}

\section{Introduction}
\label{sec:intro}

Large language models (LLMs) generate text one token at a time, with a full forward pass of the model for each new token \citep{brown2020gpt3,grattafiori2024llama3,yang2025qwen3}. Speculative decoding (SD) reduces this sequential cost without changing the output distribution \citep{leviathan2023fast,chen2023accelerating,arxiv-2401-07851,arxiv-2502-19732}. In each round, a lightweight drafter proposes several future tokens that the target model then verifies in a single forward pass. Block drafters such as DFlash \citep{chen2026dflash} make this drafting step particularly cheap. DFlash conditions a small block-diffusion head \citep{arxiv-2502-09992,arxiv-2503-09573} on the hidden features of the target model and, in one forward pass, returns a probability distribution over the vocabulary for every position of a future block.

In standard decoding, however, most of this output is discarded. The decoder keeps the most likely token at each position, verifies the resulting chain, and drops every other candidate that the drafter has already scored \citep{stern2018blockwise,arxiv-2408-05636,arxiv-2504-18583,zhang2026dflare}. As Figure~\ref{fig:framework} shows, if the target rejects an early token of the chain, the round commits only a few tokens, even when the token the target wanted was among the dropped candidates. A candidate tree can retain such alternatives, and tree verification checks the entire tree in one target pass by letting each token attend only to its own ancestors \citep{miao2024specinfer,cai2024medusa}.

To build such a tree, a decoder must first choose its width. How many candidates should the tree hold? Existing answers target autoregressive drafters such as the EAGLE family \citep{li2024eagle,li2024eagle2,li2025eagle3}, which produce candidates through dependent draft steps. Adaptive and hardware-aware tree construction accounts for this drafting cost \citep{li2024eagle2,wang2025opt,chen2024sequoia}. These choices do not transfer directly to a block drafter, because one pass has already scored every candidate and a wider tree adds verification work but no drafting work.

The price of this extra verification, in turn, depends on where the model runs. It differs across GPUs, grows with the number of concurrent requests, and on some GPUs jumps once the tree crosses a kernel boundary. The best width therefore changes across deployments. Finding it by sweeping widths on every deployment is costly. A concurrent study also weighs the benefit of tree nodes against their cost \citep{wang2026smart}, while another line retrains or extends block drafters to capture dependencies between positions \citep{arxiv-2606-18394,arxiv-2607-06763}. We instead keep the drafter fixed and ask how its tree width should follow the deployment.

\begin{figure}[t]\centering
\includegraphics[width=\linewidth]{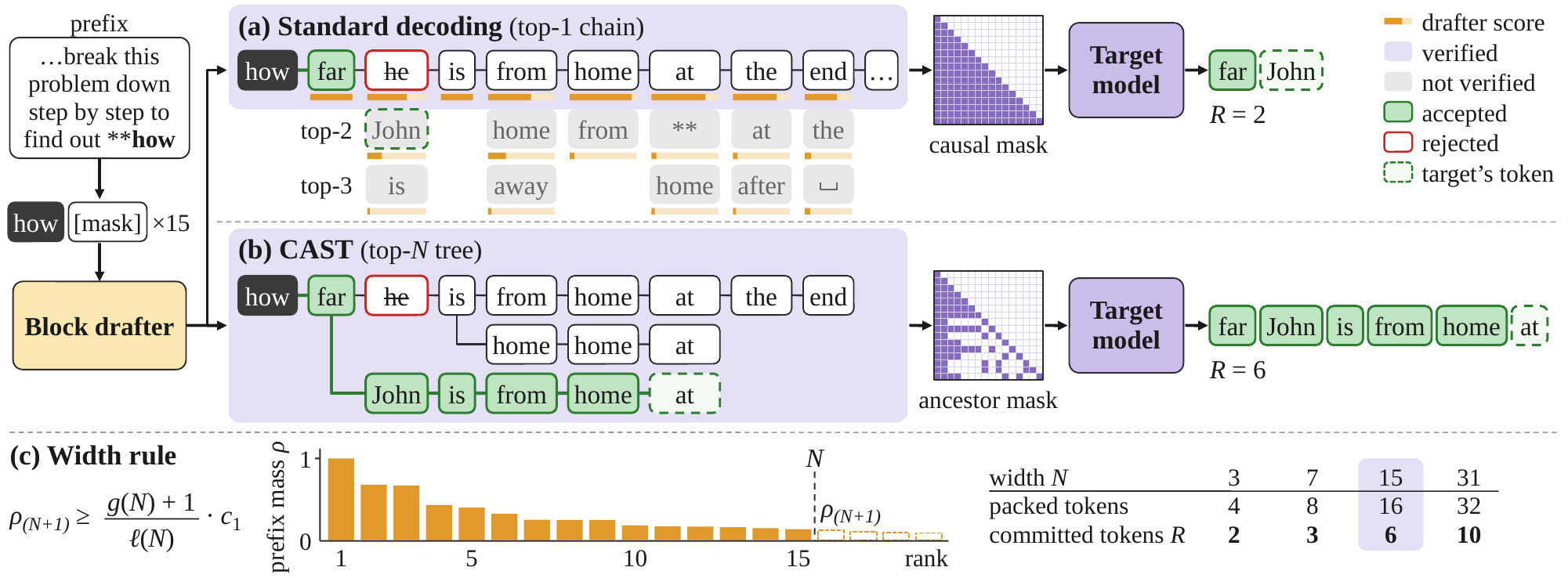}
\caption{One decoding round on a GSM8K prompt. With the same drafter scores and 16 packed tokens, (a) standard DFlash decoding commits two tokens, whereas (b) \method verifies a tree and commits six. (c) Theorem~\ref{thm:width} chooses the width $N$.}
\label{fig:framework}
\end{figure}

We introduce \method (Cost-Aware Speculative Trees), a decoding policy that turns the discarded candidates of a block drafter into a verified tree. In each round, \method runs the unchanged drafter once, scores candidate sequences by the product of the drafter's probabilities along each sequence, and keeps the highest-scoring sequences as a tree for the target to verify in one pass. To choose the width, \method applies a cost-aware stopping rule. The rule adds the next candidate as long as its chance of acceptance exceeds the number of tokens the decoder would produce, at its current speed, during the verification time that the candidate adds. Because this test needs only the drafter's scores and a short latency measurement, the width adapts to each deployment without a sweep. Under serving load, the same scores divide one shared verification budget among concurrent requests, a problem studied for autoregressive drafters \citep{arxiv-2604-09603}. We evaluate \method with DFlash as the block drafter.

We summarize our contributions and findings as follows.
\begin{itemize}
\item In \S\ref{sec:method}, we show that keeping the highest-scoring candidates maximizes the expected number of accepted tokens under the drafter's scores for any width. We then derive the cost-aware stopping rule, extend it to a budget shared by requests, and prove that \method preserves the target output distribution under greedy and sampled decoding.
\item Experiments in \S\ref{sec:exp} show that \method is faster than the standard DFlash chain in all eight hardware and model settings, with average gains of 20--36\% and up to 43\% on a single domain. Its predicted width matches the best swept width in six settings. In offline replay, the tree also lengthens accepted rounds for a second one-pass drafter \citep{arxiv-2601-19278}.
\item We observe that the best width depends strongly on the deployment. Where verification cost jumps at a kernel boundary, a 128-token tree is only $2.4\%$ faster than the standard chain, whereas the tree at the predicted width is $19.8\%$ faster. Under serving load, the best static width shrinks from a wide tree to the chain as concurrency grows, and best-first allocation of a shared budget lifts goodput over the best static width by up to $13.7\%$.
\end{itemize}

\section{Background and Setup}
\label{sec:setup}

\paragraph{Speculative decoding.}
Let $p$ denote the target model over vocabulary $\Sigma$, and let $\circ$ denote concatenation. Speculative decoding runs in rounds, and each round commits the accepted draft tokens together with one token selected by the target. If $A$ draft tokens are accepted, the round commits $R=A+1$ tokens, matching the acceptance length $\tau$ reported by DFlash. Decoding throughput is the expected committed length $\Eround$ divided by the round time \citep{leviathan2023fast,arxiv-2408-11049}. A decoder therefore gets faster by committing more tokens in each round or by making rounds cheaper.

\paragraph{Round state.}
A round starts from two pieces of state. The \emph{cache-backed prefix} $x$ holds the tokens whose target key-value (KV) cache and DFlash feature cache are already computed. The \emph{pending root} $b$ is the token committed by the previous round but not yet processed by the target. It serves as the first input of the next draft block and as the root of the next candidate tree.

\paragraph{One-pass block drafters.}
We use only the public DFlash interface \citep{chen2026dflash}. Conditioned on the cached features for $x$, the DFlash head consumes a block $[b,\mathrm{mask},\ldots,\mathrm{mask}]$ of $L{+}1{=}16$ tokens and, in one forward pass, returns a distribution $q_j(\cdot\mid x,b)$ for each future position $j\le L$. The drafter is not called again within the block. Standard DFlash decoding keeps the most likely token at each position and verifies the resulting chain. \method instead uses these distributions as candidate scores and ranks a candidate sequence $y_{1:d}$ of length $d\le L$ by the \emph{plug-in} score
\begin{equation}
\hat\pi(y_{1:d}\mid x,b)=\prod_{j\le d}q_j(y_j\mid x,b).
\label{eq:plugin}
\end{equation}
This product serves only to rank candidates. It does not assume that the target distribution factorizes over positions, and \S\ref{sec:exp} measures how well it ranks them.

\paragraph{Candidate trees and verification.}
A depth-$d$ node is a candidate sequence $y_{1:d}$ with $1\le d\le L$. A candidate tree is \emph{prefix-closed}, meaning that it contains every ancestor of each of its nodes, and its width or budget $N$ counts its \emph{nonroot} nodes. Verifying a budget-$N$ tree therefore packs $N{+}1$ tokens, the pending root included, into one target pass. Following SpecInfer and Medusa \citep{miao2024specinfer,cai2024medusa}, this pass uses an attention mask under which each node sees only the prefix and its own ancestors, and Lemma~\ref{lem:mask} shows that it reproduces exactly what $N{+}1$ separate autoregressive forward passes would compute. Under greedy decoding, the round commits the longest tree path that agrees with the target and appends the target's next token, which becomes the new pending root.

\paragraph{Sampled decoding.}
Our optimality analysis in \S\ref{sec:method} uses greedy verification. For temperature $T{>}0$, \method runs a \emph{proposal-aware stochastic tree verifier}, which applies recursive speculative sampling without replacement to the retained siblings at each node. By Theorem~\ref{thm:stoch}, its committed tokens follow the distribution of autoregressive sampling from the target at temperature $T$.

\section{The \method Width Policy}
\label{sec:method}

\method changes the verifier input while keeping the target, drafter weights, feature-cache contract, and decoding rule fixed.

\subsection{The cost of width}
\label{sec:cost}

\paragraph{The width trade-off.}
Adding a candidate node can raise the committed round length $R$ but adds selection, heap, mask, packing, target-forward, and KV-gather work. A wider tree pays off only when the gain in $R$ outweighs that work. The component probe of Table~\ref{tab:cost} isolates the target-forward part of that cost, and Theorem~\ref{thm:width} turns it into a width $N^\ast$.

\paragraph{Draft side.}
One drafter pass already scores all $L$ future positions. Widening a tree within that block adds timed host-side work but no drafter call. Depth beyond $L$ needs another draft step.

\paragraph{Verify side.}
The target-forward latency $\ell_{\mathrm{tgt}}(n)$ in Figure~\ref{fig:cost} and Table~\ref{tab:cost} grows only locally with the packed-token count $n=N{+}1$. From $n{=}16$ to $n{=}64$ it rises only $+12\%$ on Blackwell-8B and a few percent on every other GPU and target pair, then jumps at the Blackwell-8B $n{=}128$ kernel/tile boundary alone. Over $n\in[16,128]$ a linear fit $\ell_{\mathrm{tgt}}(n)\approx c_0 + c^{\mathrm{fw}}_1 n$ with intercept $c_0$ and forward slope $c^{\mathrm{fw}}_1$ approximates the measurements, with probe details in Appendix~\ref{app:more}.

\begin{figure}[t]
\begin{minipage}[t]{0.48\linewidth}
\centering
\includegraphics[width=\linewidth]{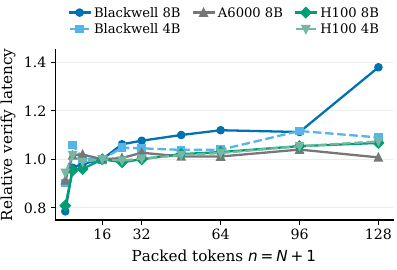}
\caption{Batch-one target-forward latency relative to 16 packed tokens across GPUs and targets, with packed count $n=N+1$.}
\label{fig:cost}
\end{minipage}\hfill
\begin{minipage}[t]{0.48\linewidth}
\centering
\includegraphics[width=\linewidth]{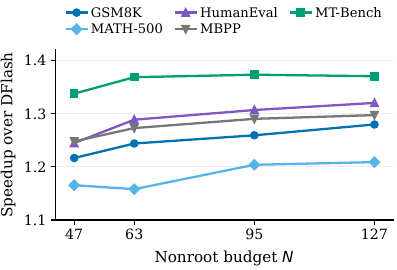}
\caption{Qwen3-8B width sweeps on H100 SXM across five domains. Speedup normalizes to standard DFlash using decode-only latency.}
\label{fig:h100sweep}
\end{minipage}
\end{figure}

\subsection{Budget-optimal candidate trees}
\label{sec:theory}
\label{sec:tree-belief}

The results here are round-local and greedy. Using the convention of \S\ref{sec:setup}, a \emph{candidate tree} is a prefix-closed set $S$ of nonempty strings after the root $b$, with depth at most $L$ and budget $N=|S|$. Under greedy verification, let $Y$ be the target's greedy continuation after $b$. The accepted-prefix length, the committed round length, and the target prefix law of a node $v$ are
\[
A(S)=\max\{k: Y_{1:k}\in S\},\qquad R(S)=A(S)+1,\qquad \pi^\star(v)=\Pr\!\left[Y_{1:|v|}=v\mid x,b\right].
\]
The fixed-belief results replace $\pi^\star$ by any prefix belief $\rho(v)$, nonincreasing along prefixes, and write $\mathbb{E}_{\rho}$ for expectation under a law with these prefix masses. By Lemma~\ref{lem:nodesum}, expected accepted length then decomposes into $\sum_{v\in S}\rho(v)$, which turns tree selection into a top-$N$ choice.

\begin{theorem}[Top-prefix optimality under a fixed belief]
\label{thm:tree}
Let $\mathcal{U}$ be a fixed finite prefix-closed universe of feasible nonroot strings ($N\le|\mathcal{U}|$) and let $\rho$ be prefix masses, nonincreasing from ancestors to descendants. Among prefix-closed $S\subseteq\mathcal{U}$ with $|S|=N$, $\mathbb{E}_{\rho}[A(S)]$ is maximized by the $N$ strings in $\mathcal{U}$ with largest $\rho$ (ties toward prefixes), and this set is itself prefix-closed.
\end{theorem}

By Corollary~\ref{cor:plugin-tree}, setting $\rho=\hat\pi$ from \eqref{eq:plugin} gives the exact top-$N$ tree for the plug-in objective inside the retained universe, under the one-pass marginal ranking interface of Assumption~\ref{asm:indep}, which describes the drafter and not the target law. The public decoder keeps the top $K{=}8$ tokens at each future position and expands the tree best-first in decreasing $\hat\pi$ order, a rank cap that matches the full-rank tree's accepted length to within $1.1\%$ at the evaluated budgets, as measured in Appendix~\ref{app:rankcap}.

Proposition~\ref{prop:regret} bounds the gap between the retained-universe oracle tree and the plug-in tree by a sum of prefix-belief estimation errors over the symmetric difference of the two trees. The drafter is well calibrated at the first tree level, as reported in Appendix~\ref{app:analysis}.

\subsection{Cost-aware width}
\label{sec:width}

Theorem~\ref{thm:tree} says which nodes to include once $N$ is fixed. The cost slope determines how large $N$ should be. Write $g(N)=\mathbb{E}_{\rho}[A(S_N)]$ for the expected accepted length of the optimal budget-$N$ tree $S_N$ over a fixed finite prefix-closed candidate universe, and let the round time be
\[
\ell(N)\;=\;c_{\mathrm{draft}} + h(N) + \ell_{\mathrm{tgt}}(n_0+N),
\]
where $c_{\mathrm{draft}}$ is the drafter forward cost, $h(N)$ is packing, mask, heap, and KV-gather overhead, $n_0=1$ counts the always-verified root, and $\ell_{\mathrm{tgt}}$ is the target-forward latency of \S\ref{sec:cost}.

\begin{theorem}[Marginal stopping rule]
\label{thm:width}
The function $g$ is concave in $N$ with increments $g(N)-g(N-1)=\rho_{(N)}$, where $\rho_{(N)}$ is the $N$-th largest prefix mass in the universe (an order statistic).
Assume $h(N)+\ell_{\mathrm{tgt}}(n_0+N)$ is locally affine with nonnegative total slope $c_1$ over the allowed range.
Then throughput $(g(N)+1)/\ell(N)$ is unimodal, and from budget $N$, adding node $N{+}1$ does not reduce it if and only if
\[
\underbrace{\rho_{(N+1)}}_{\text{prefix mass of next node}} \;\ge\; \underbrace{\frac{g(N)+1}{\ell(N)}}_{\text{throughput (tok/ms)}}\cdot \underbrace{c_1}_{\text{cost/node (ms)}}.
\]
A maximizer is found by adding nodes until this condition first fails, with boundary handling at the smallest and largest allowed budgets.
\end{theorem}

The rule keeps adding candidates while the next node's prefix mass outweighs the system's token-time exchange rate. In our experiments we set $c_1$ to the forward-only slope $c^{\mathrm{fw}}_1$ that \S\ref{sec:cost} fits to the probe measurements. Because the probe omits the overhead $h$, whose slope is nonnegative, this choice under-estimates $c_1$, and the sweeps of \S\ref{sec:exp} bound the resulting error.

\paragraph{Scope.}
The affine assumption is local. For example, hybrid linear-attention targets pay a fixed chunk-kernel entry cost that the chain and the tree share, after which their cost is again locally affine, as Appendix~\ref{app:hybrid_cost} shows.

\paragraph{Empirical validation.}
On Blackwell/Qwen3-8B the threshold is about $0.015$, crossed between $N{=}47$ and $63$, so we deploy $N^\ast{=}63$. The same rule selects wider trees on flatter cost curves. Appendix~\ref{app:more} gives the cost slopes and sweep boundaries.

\paragraph{Width under serving load.}
\label{sec:serving-theory}
Under load, $B$ concurrent requests share one verifier budget $M$, the sum of their widths. One drafter pass serves the batch at any width, while verification grows with the total packed count. Because expected accepted length is additive across requests, global best-first expansion maximizes it over the disjoint candidate forests, as Proposition~\ref{prop:alloc} proves. The batched stopping rule of Proposition~\ref{prop:serve-stop} adds the next global node while its mass reaches the batch token-time exchange rate. Both results reduce to the batch-one rule at $B{=}1$ and permit unequal widths.

\subsection{The \method decoder}
\label{sec:decoder}

\method is a round-level decoder for the unmodified one-pass DFlash interface, whose DFlash-specific part is the pending-root round contract of \S\ref{sec:setup}. Given the cache-backed prefix $x$, pending root $b$, feature cache, and budget $N$, each round runs the unchanged DFlash drafter once on the block $[b,\mathrm{mask},\ldots,\mathrm{mask}]$, keeps the top $K$ tokens at each future position and enumerates the top-$N$ prefix scores of Theorem~\ref{thm:tree} inside that universe. It verifies the packed tree in one ancestor-masked target pass, then walks the tree, emits the accepted path and the next pending root, and rebuilds the caches, with the full algorithm in Appendix~\ref{app:impl}.

\paragraph{Exactness and sampling.}
At temperature 0 the walk commits the target greedy continuation by construction. In fp32, \method and standard DFlash both reproduce autoregressive decoding token for token on Qwen3-8B and Qwen3-4B. In bf16, both depart from autoregressive decoding only at near-ties of the top two target logits, where shape-dependent rounding can flip the greedy token.

For $T{>}0$ the proposal-aware stochastic verifier of \S\ref{sec:setup} instantiates SpecInfer/SpecTr-style recursive speculative sampling \citep{miao2024specinfer,sun2023spectr,sun2024block,weng2025traversal,zhou2026hsd} over the prefix-closed tree, a rule that Theorem~\ref{thm:stoch} proves distribution-preserving. The total variation of its committed tokens from autoregressive sampling is below the Monte-Carlo floor in fp64 and about $0.004$ in bf16, as Appendix~\ref{app:stats} details.
\section{Experiments}
\label{sec:exp}

\paragraph{Setup.}
Our main experiments use the public, unmodified z-lab DFlash-b16 heads for non-thinking Qwen3-4B and Qwen3-8B \citep{yang2025qwen3}, bf16, batch~1, and greedy decoding except at $T{=}1$.

The rule thresholds drafter-score order statistics from $1{,}500$ offline GSM8K rounds, using one curve across deployments and changing only the measured cost slope. Width selection uses scores and costs, not timing-sweep outcomes. Appendix~\ref{app:stats} describes the paired prompt-level bootstrap.

We time autoregressive (AR) decoding, DFlash, and \method in one shared decoding loop on an H100 SXM across five math, code, and chat domains, using 40 prompts in each domain and the same token budget for every method. Runs on Blackwell (RTX~PRO~6000 Server Edition) and A6000 cover GSM8K, MT-Bench, and HumanEval. We report decode-only steady-state time $t$ in ms/token, speedup $t_{\mathrm{AR}}/t$, and gain $t_{\mathrm{DFlash}}/t_{\mathrm{CAST}}-1$ over DFlash, averaging domain ratios for each setting. Appendices~\ref{app:more} and~\ref{app:h100sxm} give timer boundaries, datasets, and budgets.

\paragraph{Speedup at the predicted width.}
In Table~\ref{tab:h100main}, \method improves on DFlash across every H100 domain for Qwen3-8B and Qwen3-4B. Both targets average $5.62\times$ speedup over AR, with longer committed rounds throughout.

EAGLE-3 uses its official code and public Qwen3 heads at the reported settings, and it shares the machine, prompts, decode-only timer, and AR baseline of the DFlash and \method rows. The same cost rule selects $N^\ast{=}127$ for the block-10 head of LLaMA-3.1-8B \citep{grattafiori2024llama3}, extending the benefit to a second target family and block size. Appendix~\ref{app:llama} gives that sweep.

The nearly flat H100 SXM cost curve keeps the stopping threshold below candidate mass through the largest evaluated budget, and the sweeps in Figure~\ref{fig:h100sweep} peak at or near that budget in every domain.

\begin{table}[!ht]
\caption{Qwen3 decoding on H100 SXM. Cells report speedup over AR and committed round length $R$. Parentheses give baseline packed tokens or the \method nonroot budget.}
\label{tab:h100main}
\centering
\footnotesize
\setlength{\tabcolsep}{1.5pt}
\begin{tabular}{c l cc cc cc cc cc cc}
\toprule
\multirow{2}[2]{*}{Model} & \multirow{2}[2]{*}{Method} & \multicolumn{4}{c}{\textsc{Math}} & \multicolumn{4}{c}{\textsc{Code}} & \multicolumn{2}{c}{\textsc{Chat}} & \\
\cmidrule(lr){3-6}\cmidrule(lr){7-10}\cmidrule(lr){11-12}
 & & \multicolumn{2}{c}{GSM8K} & \multicolumn{2}{c}{MATH-500} & \multicolumn{2}{c}{HumanEval} & \multicolumn{2}{c}{MBPP} & \multicolumn{2}{c}{MT-Bench} & \multicolumn{2}{c}{\textit{Avg.}}\\
\midrule
\rowcolor{grouprow} \multicolumn{2}{c}{Temperature $=0$} & \tabsmall Speedup & \tabsmall $R$ & \tabsmall Speedup & \tabsmall $R$ & \tabsmall Speedup & \tabsmall $R$ & \tabsmall Speedup & \tabsmall $R$ & \tabsmall Speedup & \tabsmall $R$ & \tabsmall Speedup & \tabsmall $R$\\
\midrule
 & EAGLE-3 \hfill\tabsmall\secondarycolor{(16)}\hspace{-\fill} & 2.58$\times$ & 4.45 & 2.37$\times$ & 4.11 & 2.23$\times$ & 3.97 & 2.37$\times$ & 4.11 & 2.07$\times$ & 3.74 & 2.32$\times$ & 4.08\\
 & EAGLE-3 \hfill\tabsmall\secondarycolor{(60)}\hspace{-\fill} & 3.08$\times$ & 5.34 & 2.79$\times$ & 4.88 & 2.69$\times$ & 4.76 & 2.80$\times$ & 4.93 & 2.41$\times$ & 4.49 & 2.75$\times$ & 4.88\\
  & \cellcolor{baselinerow}DFlash \hfill\tabsmall\secondarycolor{(16)}\hspace{-\fill} &\cellcolor{baselinerow}4.54$\times$ & \cellcolor{baselinerow}6.23 & \cellcolor{baselinerow}5.57$\times$ & \cellcolor{baselinerow}7.81 & \cellcolor{baselinerow}4.67$\times$ & \cellcolor{baselinerow}6.52 & \cellcolor{baselinerow}4.82$\times$ & \cellcolor{baselinerow}6.56 & \cellcolor{baselinerow}2.30$\times$ & \cellcolor{baselinerow}4.06 & \cellcolor{baselinerow}4.38$\times$ & \cellcolor{baselinerow}6.24\\
\multirow{-4}{*}{Q3-8B} & \cellcolor{oursrow}\method \hfill\tabsmall\secondarycolor{($N^\ast{=}127$)}\hspace{-\fill} & \cellcolor{oursrow}\textbf{5.81$\times$} & \cellcolor{oursrow}\textbf{8.46} & \cellcolor{oursrow}\textbf{6.74$\times$} & \cellcolor{oursrow}\textbf{9.87} & \cellcolor{oursrow}\textbf{6.16$\times$} & \cellcolor{oursrow}\textbf{9.04} & \cellcolor{oursrow}\textbf{6.25$\times$} & \cellcolor{oursrow}\textbf{9.09} & \cellcolor{oursrow}\textbf{3.16$\times$} & \cellcolor{oursrow}\textbf{5.52} & \cellcolor{oursrow}\textbf{5.62$\times$} & \cellcolor{oursrow}\textbf{8.40}\\
\addlinespace[2pt]
 & EAGLE-3 \hfill\tabsmall\secondarycolor{(16)}\hspace{-\fill} & 1.81$\times$ & 3.29 & 1.79$\times$ & 3.11 & 1.62$\times$ & 3.00 & 1.71$\times$ & 3.03 & 1.59$\times$ & 2.84 & 1.70$\times$ & 3.05\\
 & EAGLE-3 \hfill\tabsmall\secondarycolor{(60)}\hspace{-\fill} & 2.08$\times$ & 3.81 & 2.03$\times$ & 3.61 & 1.86$\times$ & 3.41 & 1.94$\times$ & 3.46 & 1.87$\times$ & 3.34 & 1.96$\times$ & 3.53\\
  & \cellcolor{baselinerow}DFlash \hfill\tabsmall\secondarycolor{(16)}\hspace{-\fill} &\cellcolor{baselinerow}4.40$\times$ & \cellcolor{baselinerow}6.13 & \cellcolor{baselinerow}5.34$\times$ & \cellcolor{baselinerow}7.49 & \cellcolor{baselinerow}4.61$\times$ & \cellcolor{baselinerow}6.49 & \cellcolor{baselinerow}4.96$\times$ & \cellcolor{baselinerow}6.72 & \cellcolor{baselinerow}2.30$\times$ & \cellcolor{baselinerow}3.95 & \cellcolor{baselinerow}4.32$\times$ & \cellcolor{baselinerow}6.15\\
\multirow{-4}{*}{Q3-4B} & \cellcolor{oursrow}\method \hfill\tabsmall\secondarycolor{($N^\ast{=}127$)}\hspace{-\fill} & \cellcolor{oursrow}\textbf{5.79$\times$} & \cellcolor{oursrow}\textbf{8.47} & \cellcolor{oursrow}\textbf{6.62$\times$} & \cellcolor{oursrow}\textbf{9.76} & \cellcolor{oursrow}\textbf{6.04$\times$} & \cellcolor{oursrow}\textbf{9.02} & \cellcolor{oursrow}\textbf{6.34$\times$} & \cellcolor{oursrow}\textbf{9.11} & \cellcolor{oursrow}\textbf{3.29$\times$} & \cellcolor{oursrow}\textbf{5.78} & \cellcolor{oursrow}\textbf{5.62$\times$} & \cellcolor{oursrow}\textbf{8.43}\\
\midrule
\rowcolor{grouprow} \multicolumn{2}{c}{Temperature $=1$} & \tabsmall Speedup & \tabsmall $R$ & \tabsmall Speedup & \tabsmall $R$ & \tabsmall Speedup & \tabsmall $R$ & \tabsmall Speedup & \tabsmall $R$ & \tabsmall Speedup & \tabsmall $R$ & \tabsmall Speedup & \tabsmall $R$\\
\midrule
 & EAGLE-3 \hfill\tabsmall\secondarycolor{(16)}\hspace{-\fill} & 2.22$\times$ & 4.30 & 2.32$\times$ & 3.91 & 2.19$\times$ & 3.90 & 2.22$\times$ & 3.89 & 1.82$\times$ & 3.52 & 2.15$\times$ & 3.90\\
 & EAGLE-3 \hfill\tabsmall\secondarycolor{(60)}\hspace{-\fill} & 2.63$\times$ & 5.17 & 2.65$\times$ & 4.62 & 2.52$\times$ & 4.57 & 2.58$\times$ & 4.66 & 2.16$\times$ & 4.29 & 2.51$\times$ & 4.66\\
  & \cellcolor{baselinerow}DFlash \hfill\tabsmall\secondarycolor{(16)}\hspace{-\fill} &\cellcolor{baselinerow}4.11$\times$ & \cellcolor{baselinerow}5.69 & \cellcolor{baselinerow}4.39$\times$ & \cellcolor{baselinerow}6.36 & \cellcolor{baselinerow}4.08$\times$ & \cellcolor{baselinerow}5.77 & \cellcolor{baselinerow}3.83$\times$ & \cellcolor{baselinerow}5.31 & \cellcolor{baselinerow}2.15$\times$ & \cellcolor{baselinerow}3.63 & \cellcolor{baselinerow}3.71$\times$ & \cellcolor{baselinerow}5.35\\
\multirow{-4}{*}{Q3-8B} & \cellcolor{oursrow}\method \hfill\tabsmall\secondarycolor{($N^\ast{=}127$)}\hspace{-\fill} & \cellcolor{oursrow}\textbf{5.16$\times$} & \cellcolor{oursrow}\textbf{7.93} & \cellcolor{oursrow}\textbf{5.39$\times$} & \cellcolor{oursrow}\textbf{8.46} & \cellcolor{oursrow}\textbf{5.35$\times$} & \cellcolor{oursrow}\textbf{8.27} & \cellcolor{oursrow}\textbf{4.97$\times$} & \cellcolor{oursrow}\textbf{7.56} & \cellcolor{oursrow}\textbf{2.94$\times$} & \cellcolor{oursrow}\textbf{5.30} & \cellcolor{oursrow}\textbf{4.76$\times$} & \cellcolor{oursrow}\textbf{7.50}\\
\addlinespace[2pt]
 & EAGLE-3 \hfill\tabsmall\secondarycolor{(16)}\hspace{-\fill} & 1.60$\times$ & 3.21 & 1.68$\times$ & 3.00 & 1.67$\times$ & 3.00 & 1.64$\times$ & 2.98 & 1.38$\times$ & 2.80 & 1.59$\times$ & 3.00\\
 & EAGLE-3 \hfill\tabsmall\secondarycolor{(60)}\hspace{-\fill} & 1.83$\times$ & 3.75 & 1.89$\times$ & 3.48 & 1.83$\times$ & 3.41 & 1.81$\times$ & 3.42 & 1.53$\times$ & 3.22 & 1.78$\times$ & 3.46\\
  & \cellcolor{baselinerow}DFlash \hfill\tabsmall\secondarycolor{(16)}\hspace{-\fill} &\cellcolor{baselinerow}4.16$\times$ & \cellcolor{baselinerow}5.80 & \cellcolor{baselinerow}4.54$\times$ & \cellcolor{baselinerow}6.55 & \cellcolor{baselinerow}4.27$\times$ & \cellcolor{baselinerow}5.96 & \cellcolor{baselinerow}4.29$\times$ & \cellcolor{baselinerow}5.96 & \cellcolor{baselinerow}2.21$\times$ & \cellcolor{baselinerow}3.79 & \cellcolor{baselinerow}3.90$\times$ & \cellcolor{baselinerow}5.61\\
\multirow{-4}{*}{Q3-4B} & \cellcolor{oursrow}\method \hfill\tabsmall\secondarycolor{($N^\ast{=}127$)}\hspace{-\fill} & \cellcolor{oursrow}\textbf{5.14$\times$} & \cellcolor{oursrow}\textbf{7.88} & \cellcolor{oursrow}\textbf{5.56$\times$} & \cellcolor{oursrow}\textbf{8.69} & \cellcolor{oursrow}\textbf{5.54$\times$} & \cellcolor{oursrow}\textbf{8.59} & \cellcolor{oursrow}\textbf{5.41$\times$} & \cellcolor{oursrow}\textbf{8.27} & \cellcolor{oursrow}\textbf{3.02$\times$} & \cellcolor{oursrow}\textbf{5.44} & \cellcolor{oursrow}\textbf{4.93$\times$} & \cellcolor{oursrow}\textbf{7.77}\\
\bottomrule
\end{tabular}

\end{table}

\begin{table}[t]
\caption{Predicted and sweep-oracle widths across the eight settings. Gain over DFlash is averaged over domains. Gap is in percentage points. $\ddagger$ marks cells measured on a second A6000 host.}
\label{tab:rule}
\centering
\small
\setlength{\tabcolsep}{3pt}
\begin{tabular}{l >{\columncolor{oursrow}}r >{\columncolor{oursrow}}r rr r r}
\toprule
\rowcolor{grouprow}
Setting (drafter block) & $N^\ast$ & Gain at $N^\ast$ & \multicolumn{2}{c}{Oracle width (gain)} & Gap & Gain at $N{=}127$\\
\midrule
H100 SXM, Qwen3-8B (b16) & 127 & \textcolor{posdelta}{+29.5\%} & 127 & \textcolor{posdelta}{+29.5\%} & +0.0 & \textcolor{posdelta}{+29.5\%}\\
H100 SXM, Qwen3-4B (b16) & 127 & \textcolor{posdelta}{+31.5\%} & 127 & \textcolor{posdelta}{+31.5\%} & +0.0 & \textcolor{posdelta}{+31.5\%}\\
Blackwell, Qwen3-8B (b16) & 63 & \textcolor{posdelta}{+19.8\%} & 95 & \textcolor{posdelta}{+21.5\%} & -1.7 & \textcolor{posdelta}{+2.4\%}\\
Blackwell, Qwen3-4B (b16) & 95 & \textcolor{posdelta}{+27.3\%} & 127 & \textcolor{posdelta}{+29.0\%} & -1.7 & \textcolor{posdelta}{+29.0\%}\\
A6000, Qwen3-8B (b16) & 95 & \textcolor{posdelta}{+28.7\%} & 95 & \textcolor{posdelta}{+28.7\%} & +0.0 & \textcolor{posdelta}{+25.1\%}$^\ddagger$\\
A6000, Qwen3-4B (b16) & 95 & \textcolor{posdelta}{+32.6\%} & 95 & \textcolor{posdelta}{+32.6\%} & +0.0 & \textcolor{posdelta}{+31.7\%}$^\ddagger$\\
H100 SXM, LLaMA-3.1-8B (b10) & 127 & \textcolor{posdelta}{+27.4\%} & 127 & \textcolor{posdelta}{+27.4\%} & +0.0 & \textcolor{posdelta}{+27.4\%}\\
H100 SXM, Qwen3-Coder-30B-A3B (b16) & 127 & \textcolor{posdelta}{+35.6\%} & 127 & \textcolor{posdelta}{+35.6\%} & +0.0 & \textcolor{posdelta}{+35.6\%}\\
\bottomrule
\end{tabular}

\end{table}

\paragraph{Width across hardware.}
The Blackwell-8B cost cliff erodes most of the gain at the largest width. Gains plateau at $N{=}63$--$95$ and fall at $127$. Flatter A6000 and Blackwell-4B curves support wider trees. Appendix~\ref{app:more} gives their full sweeps.

Across all eight settings of Table~\ref{tab:rule}, the predicted gain is within $1.7$ percentage points of the sweep oracle. The sweeps use $N\in\{15,31,47,63,95,127\}$. Among them, the H100 SXM Qwen3-8B and Qwen3-4B sweeps start at $N{=}47$, and the A6000 sweeps end at $N{=}95$. The settings cover three GPUs, two model families, two block sizes, and a 30B sparse mixture-of-experts target.

\begin{figure}[t]
\centering
\includegraphics[width=\linewidth]{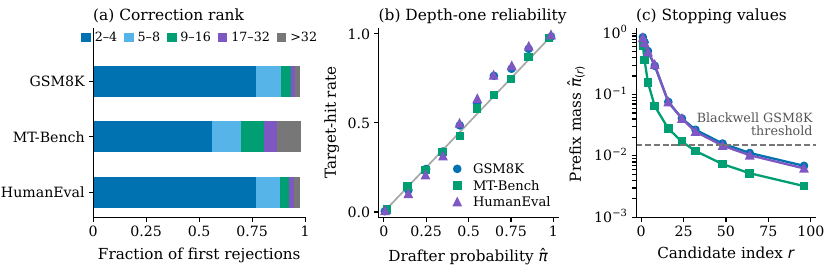}
\caption{Qwen3-8B greedy replay diagnostics on 40-prompt slices. (a) Target-correction rank at the first DFlash rejection. (b) Depth-one reliability. (c) Sorted candidate prefix masses. The dashed line is the Blackwell GSM8K threshold.}
\label{fig:analysis}
\end{figure}

\begin{wrapfigure}[20]{R}{0.48\linewidth}
\centering
\includegraphics[width=\linewidth]{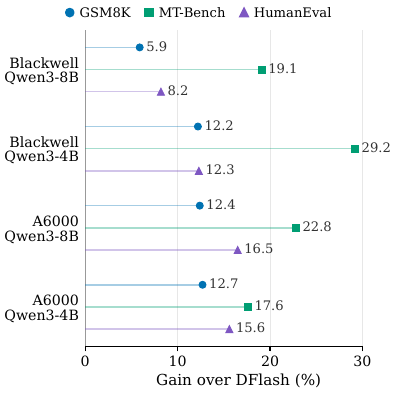}
\caption{Gain over DFlash at 16 packed tokens.}
\label{tab:shape-summary}
\end{wrapfigure}

\paragraph{Candidate scores.}
Figure~\ref{fig:analysis} uses the same 40-prompt offline replay slices as the main rows. In this replay at a fixed budget of 47 nonroot nodes, neither a fixed branching schedule nor uniform binary branching commits longer rounds than the 15-node chain in any domain, whereas the top-$N$ tree does in every domain. On GSM8K its committed round length reaches $7.7$ against $6.3$ for the chain, and Appendix~\ref{app:analysis} gives the remaining cells. At the first chain rejection, the target correction lies at draft ranks 2--4 in $56$--$77\%$ of events. Depth-one reliability and the falling prefix masses explain the ranking and stopping rule.

\paragraph{Shape at a matched budget.}
At 16 packed tokens, \method keeps $N{=}15$ nonroot nodes from the same block-16 drafter as DFlash, so the two decoders differ only in tree shape. This shape alone makes \method $5.9$ to $29.2\%$ faster than DFlash in all twelve Blackwell and A6000 cells of Figure~\ref{tab:shape-summary}. Table~\ref{tab:width-shape} gives the full round lengths and timings.

\paragraph{Sampled decoding.}
The $T{=}1$ blocks of Table~\ref{tab:h100main} keep the same rule-predicted width and replace greedy verification by the proposal-aware stochastic verifier of Theorem~\ref{thm:stoch}. Sampling shortens the committed rounds of DFlash and \method, yet the gain of \method over DFlash stays between $+22\%$ and $+37\%$ in every domain on both targets. On Blackwell with Qwen3-8B, the gain exceeds $15\%$ under each of three sampling seeds, as Table~\ref{tab:bootstrap} shows.

\paragraph{Robustness.}
On Blackwell, a paired bootstrap over expanded prompt sets places every 95\% confidence interval for the mean prompt-level gain over DFlash above $+14\%$, as Table~\ref{tab:bootstrap} shows. The gain also holds for longer generations and prefixes on H100 SXM. On MATH-500 with Qwen3-8B, \method is $21\%$ faster than DFlash at 512 new tokens and $20\%$ faster at 2048. On GovReport summarization prompts cut to 8k tokens, it is $28\%$ faster on Qwen3-8B and $27\%$ faster on Qwen3-4B, within 3 percentage points of the sweep oracle. Appendix~\ref{app:h100sxm} reports these cells.

\begin{wraptable}[16]{R}{0.48\linewidth}
\centering
\footnotesize
\setlength{\tabcolsep}{2pt}
\setlength{\abovecaptionskip}{0pt}
\caption{Mean gain over DFlash on \mbox{Blackwell} at the predicted width $N^\ast$, which is $63$ for Qwen3-8B and $95$ for Qwen3-4B. Brackets give 95\% bootstrap intervals for greedy decoding and ranges over three seeds for sampling.}
\label{tab:bootstrap}
\begin{tabular}{*{3}{>{\centering\arraybackslash}p{\dimexpr(\linewidth-6\tabcolsep)/3\relax}}}
\toprule
\rowcolor{grouprow}
GSM8K & MT-Bench & HumanEval\\
\midrule
\grouphead{3}{Qwen3-8B, greedy decoding}\\
\textcolor{posdelta}{+18.1\%} & \textcolor{posdelta}{+24.6\%} & \textcolor{posdelta}{+16.9\%}\\
{\tabsmall$[16.4, 19.8]$} & {\tabsmall$[21.6, 27.7]$} & {\tabsmall$[14.8, 19.0]$}\\
\grouphead{3}{Qwen3-4B, greedy decoding}\\
\textcolor{posdelta}{+25.7\%} & \textcolor{posdelta}{+33.2\%} & \textcolor{posdelta}{+23.2\%}\\
{\tabsmall$[24.0, 27.4]$} & {\tabsmall$[29.4, 37.0]$} & {\tabsmall$[21.1, 25.3]$}\\
\grouphead{3}{Qwen3-8B, sampling at $T{=}1$}\\
\textcolor{posdelta}{+16.7\%} & \textcolor{posdelta}{+30.1\%} & \textcolor{posdelta}{+22.3\%}\\
{\tabsmall$[16.0, 18.0]$} & {\tabsmall$[28.7, 32.3]$} & {\tabsmall$[20.3, 25.3]$}\\
\bottomrule
\end{tabular}
\end{wraptable}

\paragraph{Comparison with EAGLE-3 on LLaMA.}
In Table~\ref{tab:llama-main}, \method at the predicted width averages $3.87\times$ over AR on LLaMA-3.1-8B, against $2.59\times$ for the official EAGLE-3 with a 60-token tree. It is faster than that tree by $30$ to $60\%$ in every domain. On MT-Bench, DFlash alone trails EAGLE-3 at $2.32\times$ against $2.34\times$, while \method reaches $3.05\times$. On GSM8K, \method and the 60-token EAGLE-3 tree commit the same $R{=}6.08$, yet EAGLE-3 runs eight draft passes in each round to build its tree and \method runs one. Appendix~\ref{app:llama} reports every width.

\begin{table}[t]
\caption{LLaMA-3.1-8B greedy decoding with the block-10 DFlash head on H100 SXM. Cells report speedup over AR. Parentheses give baseline packed tokens or the \method nonroot budget.}
\label{tab:llama-main}
\centering
\small
\setlength{\tabcolsep}{4pt}
\begin{tabular}{l *{6}{w{c}{46.68pt}}}
\toprule
\multirow{2}[2]{*}{Method} & \multicolumn{2}{c}{\textsc{Math}} & \multicolumn{2}{c}{\textsc{Code}} & \textsc{Chat} & \\
\cmidrule(lr){2-3}\cmidrule(lr){4-5}\cmidrule(lr){6-6}
 & GSM8K & MATH-500 & HumanEval & MBPP & MT-Bench & \textit{Avg.}\\
\midrule
EAGLE-3 \hfill\tabsmall\secondarycolor{(10)}\hspace{-\fill} & 2.05$\times$ & 1.75$\times$ & 2.28$\times$ & 2.25$\times$ & 1.73$\times$ & 2.01$\times$\\
EAGLE-3 \hfill\tabsmall\secondarycolor{(60)}\hspace{-\fill} & 2.67$\times$ & 2.31$\times$ & 2.80$\times$ & 2.81$\times$ & 2.34$\times$ & 2.59$\times$\\
\rowcolor{baselinerow}DFlash \hfill\tabsmall\secondarycolor{(10)}\hspace{-\fill} & 3.00$\times$ & 2.76$\times$ & 3.60$\times$ & 3.56$\times$ & 2.32$\times$ & 3.05$\times$\\
\ours\method \hfill\tabsmall\secondarycolor{($N^\ast{=}127$)}\hspace{-\fill} & \textbf{3.82$\times$} & \textbf{3.59$\times$} & \textbf{4.38$\times$} & \textbf{4.49$\times$} & \textbf{3.05$\times$} & \textbf{3.87$\times$}\\
\bottomrule
\end{tabular}
\end{table}

\paragraph{Second one-pass drafter.}
In offline replay over the one-pass logits of DART \citep{arxiv-2601-19278}, a best-first tree of 63 nodes lengthens the committed round over the eight-node top-1 chain by $40\%$ on GSM8K and $34\%$ on HumanEval. Timed on an A6000 with Qwen3-8B, DART's own tree decoder runs faster with 60 nodes than with 8 in both domains, reaching a $2.14\times$ speedup over autoregressive decoding on GSM8K and $2.34\times$ on HumanEval. Appendix~\ref{app:dart} details both measurements.

\section{Serving under load}
\label{sec:serving}

No single static width serves the load range. The drafter forward is width-independent, one pass serving all $B$ requests, while the verifier width tax in Figure~\ref{fig:serving_cost}a grows from $+10\%$ at $B{=}1$ to $+237\%$ at $B{=}32$. In the batched harness of Appendix~\ref{app:serving}, the best static width is therefore a fixed wide tree up to $B{=}8$ and the plain chain from $B{=}16$. Global best-first allocation from Section~\ref{sec:serving-theory} instead splits one shared budget across requests. In offline replay of 32-request batches whose budget averages four nodes for each request, it raises total accepted tokens over a uniform split by $16.8$ to $22.5\%$. Measured on Blackwell with Qwen3-8B and a ragged verifier kernel, it lifts goodput over the best static width by a median $7.0\%$ at $B{=}16$ and $13.7\%$ at $B{=}32$ across four runs.

\begin{figure}[t]
\centering
\includegraphics[width=\linewidth]{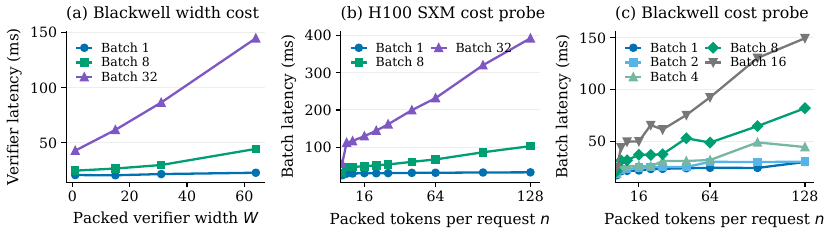}
\caption{Qwen3-8B verification cost under load. (a) Blackwell verifier latency over the packed width $W$ of each request. (b,c) Batch latency over the packed tokens $n$ of each request on H100 SXM and Blackwell.}
\label{fig:serving_cost}
\end{figure}

\section{Related Work}
\label{sec:related}
SMART applies a marginal benefit--cost rule to existing tree decoders \citep{wang2026smart}. For autoregressive drafters, OPT-Tree maximizes a draft-probability proxy for expected acceptance and stops deepening once the gain falls below a threshold tied to the cost of a draft step \citep{wang2025opt}. A one-pass drafter adds no draft step within its block, and our threshold instead prices measured verification time. Our evaluation examines deployment-specific width prediction and shared-budget allocation across requests.

A related line is draft trees over one-pass parallel logits. DART prunes such logits with an external n-gram continuity signal \citep{arxiv-2601-19278}, answering which candidates to keep under a supplied construction rule. We derive the budget-optimal tree of Theorem~\ref{thm:tree} from the drafter's own product score and ask how wide it should be on the measured hardware.

JetSpec retrains the head into a causal-parallel tree drafter \citep{arxiv-2606-18394}, while Weaver adds an autoregressive adapter over the marginal logits \citep{arxiv-2607-06763}. These approaches address conditional dependence. \method instead retains the checkpoint and prices verification. ECHO shares depth and width across requests with AR drafters \citep{arxiv-2604-09603}. For a one-pass drafter, width adds no draft call, leaving verification as the shared cost.

D-cut \citep{arxiv-2607-14647} prunes the verification depth of each request's chain under a profiled latency table, while our allocation splits prefix-closed tree width by best-first mass at measured goodput.

Beyond speculative decoding, draft-agreement routing allocates thinking budgets using agreement between inexpensive drafts \citep{lee2026dart_routing}, and adaptive correct-only rewards reduce reasoning length through efficiency-aware training \citep{lee2026acoer}. These approaches control the amount of reasoning, while \method controls verification work without changing the target output distribution.

Appendix~\ref{app:related} reviews tree verification, adaptive AR trees, and drafter alignment. Our sampled verifier follows proposal-aware verification and traversal \citep{sun2024block,zhou2026hsd,weng2025traversal,sun2023spectr,arxiv-2502-18779}. Existing controllers adapt length, block size, or shape from predicted signals \citep{arxiv-2405-19715,arxiv-2604-05417,arxiv-2505-15141,arxiv-2606-31315,arxiv-2607-19223}. Our width rule uses measured verification cost.

\subsubsection*{Ethics Statement}
\method changes inference scheduling while preserving target weights and the greedy rule. We use public checkpoints and benchmark prompts and collect no new data. Deployments retain the base model's safety profile and should keep its safeguards, because faster inference can reduce harmful generation costs.

\bibliographystyle{iclr2027_conference}
\bibliography{refs}

\appendix
\setlength{\belowcaptionskip}{5pt}
\label{app:conclusion_limitation}

\section{Conclusion and Limitations}

\subsection{Conclusion}
\label{app:conclusion}
We present \method, which recovers a prefix-closed candidate tree from one already-computed block and selects its width from measured verifier costs. Without sweeping over widths, the stopping rule predicts operating widths across three GPU generations. The predicted width ranges from 64 packed tokens for Qwen3-8B on Blackwell, where verification cost jumps at a kernel boundary, to 128 on H100 SXM, where it stays nearly flat. At the predicted width, the tree is faster than the standard chain in all eight settings. It also lengthens accepted rounds for a second one-pass drafter and extends to shared-budget allocation under load. Because the rule needs only offline drafter scores and a short latency probe, a deployment can recompute its width whenever the GPU or target changes. 

\subsection{Limitations}
\label{app:limit}
\paragraph{Evaluation scope.}
The claim is a controlled comparison that changes only the decoder within each setting, not a cross-system ranking. Our decoding experiments use dense targets of up to 8B parameters and, in Appendix~\ref{app:coder}, one 30B mixture-of-experts target with 3B active parameters.

\paragraph{Serving and timing scope.}
Our timing is the decode-only steady state defined in Appendix~\ref{app:more} and is not a serving throughput or cost claim.

\paragraph{Tree and drafter scope.}
The tree is optimal only inside the $K{=}8$ retained universe for the plug-in belief, and because the DFlash marginals are not path-conditioned, deeper branches stay exposed to calibration error. The free-width argument needs a one-pass drafter, since multi-step refinement would reintroduce a draft-side width cost. \method inherits the DFlash assumptions of a trained head and the feature-cache contract and does not address DFlash training cost or block-size changes.

\section{Proofs}
\label{app:proofs}

This appendix proves the formal results of \S\ref{sec:setup} and \S\ref{sec:method} in four steps.
\begin{itemize}
\item Appendix~\ref{app:proofs-exact} states the one-pass drafter interface and shows that one ancestor-masked pass reproduces separate autoregressive forwards, which gives exactness under greedy and sampled verification.
\item Appendix~\ref{app:proofs-tree} writes expected accepted length as a sum of node masses and proves that the top-$N$ tree is optimal, together with its plug-in corollary and regret bound.
\item Appendix~\ref{app:proofs-width} proves the marginal stopping rule of Theorem~\ref{thm:width}.
\item Appendix~\ref{app:proofs-serve} extends both results to a verifier budget shared by $B$ requests.
\end{itemize}

\subsection{Drafter interface and exact verification}
\label{app:proofs-exact}

\begin{assumption}[One-pass marginal ranking interface]
\label{asm:indep}
Conditioned on the cached prefix $x$ and pending root token $b$, a one-pass block drafter returns one marginal proposal distribution $q_j(\cdot\mid x,b)$ for each future position $j=1,\dots,L$ in the same forward pass.
These logits are computed before any within-block token value is sampled or committed.
The tree builder therefore ranks candidate prefixes by the plug-in score of \eqref{eq:plugin}.
\end{assumption}

\begin{lemma}[Packed-logit equivalence]
\label{lem:mask}
Pack the pending root $b$ and the nonroot nodes of the prefix-closed tree into one target input after the cache-backed prefix $x$, with the nodes in any ancestor-before-descendant order and each token placed at its depth-position. Run the target with an ancestor-only attention mask, under which the row of a node $v$ on path $y_{1:d}$ attends to exactly the prefix $x\circ b$, the ancestor tokens $y_{1:d-1}$ on its root-to-$v$ path, and $v$ itself, each at its absolute position. Then, in exact arithmetic, the logits read off the row of $v$ are identical to those of a standalone autoregressive forward on $x\circ b\circ y_{1:d}$. One packed pass therefore reproduces what $N{+}1$ separate forwards over the root and nonroot nodes would compute.
\end{lemma}
\begin{proof}
Under the ancestor-only mask the query at $v$'s row attends to exactly the rows of $x\circ b$, the ancestors $y_{1:d-1}$, and $v$, and each of these tokens sits at the same absolute position, hence the same positional (RoPE) phase, that it occupies in the standalone sequence $x\circ b\circ y_{1:d}$. Every attention and feed-forward sublayer therefore receives at $v$'s row the identical query and the identical attended key/value set as in the standalone forward and produces the identical hidden state layer by layer. Siblings and off-path nodes are masked and cannot influence $v$'s row. 
\end{proof}

\begin{theorem}[Distribution preservation under sampling]
\label{thm:stoch}
Fix a temperature $T>0$ and the baseline logit processors. At a verified node $v$ with prefix $x\circ b\circ y_{1:d}$, let $p_v$ be the target distribution at that prefix after temperature scaling and the fixed baseline logit processors, and let $q_v$ be any proposal positive on the retained siblings $\mathcal{C}_v$. The recursive speculative-sampling rule of \S\ref{sec:setup} starts from the residual distribution $\mu=p_v$. At each step it draws a candidate $c$ from the available-renormalized proposal $\tilde q$ and accepts it with probability
\[
\min\Big(1,\frac{\mu(c)}{\tilde q(c)}\Big).
\]
On rejection it updates the residual to
\[
\mu\;\leftarrow\;\frac{(\mu-\tilde q)_+}{\lVert(\mu-\tilde q)_+\rVert_1}
\]
and continues, and if all retained siblings are rejected it commits a draw from the final residual. Then, in exact arithmetic, the committed token is distributed exactly as $p_v$. By the packed-logit equivalence of Lemma~\ref{lem:mask} and the tower property over visited nodes, the committed-token sequence has the same law as autoregressive sampling from the target at temperature $T$.
\end{theorem}
\begin{proof}
This is recursive speculative sampling without replacement \citep{miao2024specinfer,sun2023spectr}. Consider one step with residual $\mu$ and available-renormalized proposal $\tilde q$. The candidate $c$ is drawn with probability $\tilde q(c)$ and accepted with probability $\min(1,\mu(c)/\tilde q(c))$. The mass committed to $c$ at this step and the mass routed onward on rejection are therefore
\[
\tilde q(c)\min\Big(1,\frac{\mu(c)}{\tilde q(c)}\Big)=\min\big(\mu(c),\tilde q(c)\big)
\qquad\text{and}\qquad
\mu-\min(\mu,\tilde q)=(\mu-\tilde q)_+ .
\]
The normalized onward mass is exactly the conditional-on-rejection correction law. Summing the accept and reject branches telescopes the committed-token law to $\mu$, which equals $p_v$ at the first step. Induction over the without-replacement schedule preserves it, and the final-residual draw closes the recursion when all retained siblings are exhausted. By Lemma~\ref{lem:mask}, the packed verify pass reproduces the logits that $N{+}1$ separate autoregressive forwards would compute, and $p_v$ is the true target law at each visited prefix. The tower property over committed nodes yields the sequence-level claim.
\end{proof}

\subsection{Optimal candidate trees}
\label{app:proofs-tree}

\begin{lemma}[Node-sum decomposition]
\label{lem:nodesum}
For any prefix-closed $S$, under any continuation law with prefix masses $\rho$,
\[
\mathbb{E}_{\rho}[A(S)] \;=\; \sum_{v\in S}\rho(v).
\]
\end{lemma}
\begin{proof}
Because $S$ is prefix-closed, $A(S)\ge k$ holds if and only if $Y_{1:k}\in S$. Summing tail probabilities then gives
\begin{align*}
\mathbb{E}_{\rho}[A(S)]
&=\sum_{k\ge1}\Pr_{\rho}\big[A(S)\ge k\big]
&&\text{tail-sum formula}\\
&=\sum_{k\ge1}\Pr_{\rho}\big[Y_{1:k}\in S\big]
&&\text{prefix closure of } S\\
&=\sum_{k\ge1}\,\sum_{v\in S\cap\Sigma^k}\rho(v)
&&\text{definition of } \rho\\
&=\sum_{v\in S}\rho(v). \qedhere
\end{align*}
\end{proof}

\begin{proof}[Proof of Theorem~\ref{thm:tree}]
Theorem~\ref{thm:tree} states that the $N$ strings of largest mass form a prefix-closed tree and that this tree maximizes expected accepted length. The proof shows the two parts in turn.

\emph{The top-$N$ set is prefix-closed.}
For any strings $v, vw$, the event $\{Y_{1:|vw|}=vw\}$ is contained in $\{Y_{1:|v|}=v\}$, and hence $\rho(vw)\le\rho(v)$.
Order all strings in $\mathcal{U}$ by $\rho$ descending, breaking ties in favor of shorter strings and then lexicographically.
If $v$ belongs to the first $N$ elements and $u$ is a prefix of $v$, then $u\in\mathcal{U}$ because $\mathcal{U}$ is prefix-closed.
By monotonicity, $\rho(u)\ge\rho(v)$. If equality holds, the tie-break places $u$ before $v$.
Thus every ancestor of $v$ also belongs to the first $N$ elements, and the top-$N$ set is prefix-closed.

\emph{The top-$N$ set is optimal.}
Write $\rho_{(i)}$ for the $i$-th largest mass in $\mathcal{U}$. By Lemma~\ref{lem:nodesum}, every prefix-closed $N$-element subset $S\subseteq\mathcal{U}$ satisfies
\[
\mathbb{E}_{\rho}[A(S)]=\sum_{v\in S}\rho(v)\;\le\;\sum_{i=1}^{N}\rho_{(i)},
\]
with equality for the first $N$ elements of the order above.
The previous paragraph shows that these first $N$ elements are themselves prefix-closed, so they attain the bound and are optimal among prefix-closed candidate trees.
If unrelated prefixes tie at the boundary, other optimal trees may also exist. The stated tie-break selects one prefix-closed optimum.
\end{proof}

\begin{corollary}[DFlash plug-in tree]
\label{cor:plugin-tree}
Under Assumption~\ref{asm:indep}, setting $\rho=\hat\pi$ from \eqref{eq:plugin} gives the exact top-$N$ tree for that plug-in objective inside the chosen feasible universe. Under a rank cap the guarantee is exact for the retained plug-in objective over the retained universe.
\end{corollary}
\begin{proof}
Applying Theorem~\ref{thm:tree} with $\rho=\hat\pi$ gives the exact top-$N$ tree for that belief inside that universe.
When the implementation first retains only top-$K$ tokens at each future position, the universe is the prefix-closed set of strings composed from those retained tokens. Omitted tokens are outside the feasible family, not covered by the exactness claim.
\end{proof}

\begin{proposition}[Plug-in regret]
\label{prop:regret}
Fix one round, one retained prefix universe $\mathcal{U}$, and one budget $N$.
Let $\mathcal{F}_N(\mathcal{U})$ be the prefix-closed subsets of $\mathcal{U}$ with size $N$.
Let $S^\ast$ maximize $\sum_{v\in S}\pi^\star(v)$ over $\mathcal{F}_N(\mathcal{U})$, and let $\hat S$ maximize $\sum_{v\in S}\hat\pi(v)$ over the same family.
Then the gap between the retained-universe oracle tree and the plug-in tree, evaluated under the round's target prefix law, satisfies
\[
\sum_{v\in S^\ast}\pi^\star(v)-\sum_{v\in \hat S}\pi^\star(v)\;\le\;\sum_{v\in S^\ast\,\triangle\,\hat S}\bigl|\pi^\star(v)-\hat\pi(v)\bigr|,
\]
a sum of prefix-belief estimation errors over the symmetric difference of the two trees.
\end{proposition}
\begin{proof}
Write $\Delta(v)=\pi^\star(v)-\hat\pi(v)$ for the estimation error at node $v$. Adding and subtracting the plug-in scores gives
\begin{align*}
\sum_{v\in S^\ast}\pi^\star(v)-\sum_{v\in \hat S}\pi^\star(v)
&=\underbrace{\Big(\sum_{v\in S^\ast}\hat\pi(v)-\sum_{v\in \hat S}\hat\pi(v)\Big)}_{\le\,0\ \text{because $\hat S$ is $\hat\pi$-optimal}}
+\sum_{v\in S^\ast}\Delta(v)-\sum_{v\in \hat S}\Delta(v)\\
&\le\sum_{v\in S^\ast\setminus\hat S}\Delta(v)-\sum_{v\in \hat S\setminus S^\ast}\Delta(v)\\
&\le\sum_{v\in S^\ast\,\triangle\,\hat S}\bigl|\Delta(v)\bigr|.
\end{align*}
The second line drops the nonpositive bracket and cancels the terms of $v\in S^\ast\cap\hat S$, which appear in both sums. The last line bounds each remaining term by its absolute value.
\end{proof}

\subsection{The marginal stopping rule}
\label{app:proofs-width}

\begin{proof}[Proof of Theorem~\ref{thm:width}]
Theorem~\ref{thm:width} states that throughput is unimodal in the budget and that adding the next node does not reduce it if and only if the node's mass is at least throughput times the cost slope. The proof has three steps.

\emph{Concavity of $g$.}
By Theorem~\ref{thm:tree} and Lemma~\ref{lem:nodesum}, the optimal budget-$N$ tree collects the $N$ largest masses, and therefore
\[
g(N)=\sum_{i\le N}\rho_{(i)},\qquad g(N)-g(N{-}1)=\rho_{(N)} .
\]
These increments are nonincreasing, and $g$ is concave on integer budgets. Set $g(0)=0$ and extend $g$ to $[0,\infty)$ by linear interpolation. The extension is concave and nondecreasing.

\emph{Unimodality of throughput.}
Let $\nu$ denote a continuous relaxation of the nonroot node budget.
In the operating range, the theorem assumes that the non-drafter part of the round time, $h(\nu)+\ell_{\mathrm{tgt}}(n_0+\nu)$, is locally affine with total slope $c_1\ge0$.
Equivalently, after absorbing $c_{\mathrm{draft}}$, the root-token cost, and the local intercept into a constant $\ell_0>0$, write
\[
\ell(\nu)=\ell_0+c_1 \nu .
\]
Thus $\ell(\nu)$ is positive and nondecreasing, and $\ell(N{+}1)=\ell(N)+c_1$.
Let $f(\nu)=(1+g(\nu))/\ell(\nu)$ denote throughput. Because $\ell(\nu)>0$, every level $\lambda$ satisfies
\[
f(\nu)\ge\lambda\iff 1+g(\nu)-\lambda\,\ell(\nu)\ge0 .
\]
The right-hand side is a concave function of $\nu$, whose nonnegativity region is an interval. Every superlevel set of $f$ is therefore an interval, and $f$ is quasiconcave.
A quasiconcave function restricted to integers is unimodal, because its superlevel sets restrict to integer intervals.

\emph{Stopping condition.}
Substituting $g(N{+}1)=g(N)+\rho_{(N+1)}$ and $\ell(N{+}1)=\ell(N)+c_1$ and multiplying both sides by the positive product $\ell(N)\,\ell(N{+}1)$ gives
\begin{align*}
f(N{+}1)\ge f(N)
&\iff \big(1+g(N)+\rho_{(N+1)}\big)\,\ell(N)\;\ge\;\big(1+g(N)\big)\big(\ell(N)+c_1\big)\\
&\iff \rho_{(N+1)}\,\ell(N)\;\ge\;\big(1+g(N)\big)\,c_1\\
&\iff \rho_{(N+1)}\;\ge\;f(N)\,c_1 .
\end{align*}
Adding node $N{+}1$ therefore does not reduce throughput if and only if this condition holds. By unimodality, the first budget at which it fails is a maximizer, with the smallest and largest allowed budgets as boundary cases.
\end{proof}

\subsection{Shared budget under load}
\label{app:proofs-serve}

\begin{proposition}[Shared-budget allocation is global best-first]
\label{prop:alloc}
Consider one batched round with requests $r=1,\dots,B$, each carrying a fixed finite prefix-closed candidate universe $\mathcal{U}_r$ with prefix masses $\rho_r$ that are nonincreasing from ancestors to descendants (for the one-pass interface, $\rho_r=\hat\pi_r$). Over prefix-closed selections $S_r\subseteq\mathcal{U}_r$ with shared budget $\sum_r|S_r|=M$, the expected total accepted length satisfies
\[
\mathbb{E}\Big[\sum_{r=1}^{B} A(S_r)\Big]=\sum_{r=1}^{B}\,\sum_{v\in S_r}\rho_r(v),
\]
and it is maximized by selecting the $M$ nodes of largest mass in the disjoint union $\mathcal{U}=\bigsqcup_r\mathcal{U}_r$ (ties toward prefixes). The maximizer is prefix-closed inside every request. Global best-first expansion across the batch in decreasing $\rho$ order is optimal for the shared budget.
\end{proposition}
\begin{proof}
By Lemma~\ref{lem:nodesum} applied to request $r$, $\mathbb{E}[A(S_r)]=\sum_{v\in S_r}\rho_r(v)$ for any prefix-closed $S_r$, so by linearity $\mathbb{E}[\sum_r A(S_r)]=\sum_r\sum_{v\in S_r}\rho_r(v)$. Form the disjoint-union universe $\mathcal{U}=\bigsqcup_r\mathcal{U}_r$ with masses $\rho$ equal to $\rho_r$ on the copy of $\mathcal{U}_r$. Because the copies share no strings, a selection $S=\bigcup_r S_r$ is prefix-closed in $\mathcal{U}$ if and only if each $S_r$ is prefix-closed in $\mathcal{U}_r$, and $|S|=\sum_r|S_r|=M$. Maximizing $\sum_{v\in S}\rho(v)$ over prefix-closed $S\subseteq\mathcal{U}$ with $|S|=M$ is then exactly the setting of Theorem~\ref{thm:tree} on $\mathcal{U}$, whose masses are nonincreasing from ancestors to descendants within each component and which has no cross-component edges. Restricting the maximizer to each component yields the request-level selections.
\end{proof}

\begin{proposition}[Batched stopping rule under load]
\label{prop:serve-stop}
Let $G(M)$ be the sum of the $M$ largest masses in $\mathcal{U}$, which is concave in $M$ because its increments $\rho_{(M)}$ are nonincreasing, and let the round time be
\[
\ell(M,B)=c_{\mathrm{draft}}(B)+h(M,B)+\ell_{\mathrm{tgt}}(B,\,B+M),
\]
where $c_{\mathrm{draft}}(B)$ is the batched drafter cost, independent of $M$, $h(M,B)$ is the overhead at total budget $M$, $\ell_{\mathrm{tgt}}(B,n)$ is the target-forward latency of $B$ requests carrying $n$ packed tokens in total, and $h+\ell_{\mathrm{tgt}}$ is locally affine in $M$ with nonnegative slope $c_1(B)$. Then batch throughput $(G(M)+B)/\ell(M,B)$ is unimodal, and the global node $M{+}1$ is added while its mass satisfies
\[
\rho_{(M+1)}\;\ge\;\frac{G(M)+B}{\ell(M,B)}\,c_1(B).
\]
The statement recovers Theorem~\ref{thm:width} at $B{=}1$.
\end{proposition}
\begin{proof}
By Proposition~\ref{prop:alloc} the optimal accepted mass at total budget $M$ is $G(M)=\sum_{i\le M}\rho_{(i)}$, the sum of the $M$ largest masses of $\mathcal{U}$. Its increments $\rho_{(M)}$ are nonincreasing, so $G$ is concave. The committed batch length at budget $M$ is $G(M)+B$, since the $B$ roots are always committed. Writing $\ell(M,B)=\ell_0(B)+c_1(B)\,M$ after absorbing the width-independent drafter cost $c_{\mathrm{draft}}(B)$ and the local intercept into $\ell_0(B)>0$, the quasiconcavity argument of Theorem~\ref{thm:width} applies to $f(M)=(G(M)+B)/\ell(M,B)$. Replacing $1+g(N)$ by $G(M)+B$ and $\ell(N)$ by $\ell(M,B)$ in the chain of equivalences of that proof gives
\[
f(M{+}1)\ge f(M)\iff \rho_{(M+1)}\ge f(M)\,c_1(B),
\]
which is the stated condition. 
\end{proof}

\section{Implementation details}
\label{app:impl}

This appendix details our single-device decoder.

\paragraph{The round.}
At round entry the emitted stream is $x\circ b$ while both caches still cover only $x$. One unchanged DFlash call on $[b,\mathrm{mask},\ldots,\mathrm{mask}]$ returns the marginals $q_{1:L}$. The tree builder selects $N$ nonroot nodes $v_1,\ldots,v_N$, the target verifies the packed input $[b,z_1,\ldots,z_N]$, where $z_i$ is the last token of $v_i$, under an ancestor-only mask, and the walk starts at $b$ and descends to the child carrying the target-selected token while one exists. When no child matches, the target-selected token $c$ is emitted as the next pending root. After accepting the path $y_{1:A}$ the target KV cache is compacted to $x\circ b\circ y_{1:A}$, the DFlash feature cache is appended for the same committed positions, and $c$ stays outside both caches until the next round, which restores the entry state.

\paragraph{Candidate construction.}
The decoder retains the top $K{=}8$ tokens at each future position and runs best-first expansion in decreasing plug-in score $\hat\pi(v\mid x,b)=\prod_{j\le d} q_j(y_j\mid x,b)$ of each node $v=y_{1:d}$ over that retained universe. A min-heap keyed by negative cumulative log-probability is seeded with the root's top-$K$ children, and each pop emits one node (token, depth, parent, rank, path score) and pushes its top-$K$ children at the next depth while depth stays below $L$. Scores use the original $q_j$, not a renormalized top-$K$ distribution, and the children of a depth-$d$ prefix come from the shared marginal $q_{d+1}$. The parent changes only the cumulative product. Round preprocessing is one top-$K$ over the $L\times|\Sigma|$ logits moved once to host memory, and the heap stage performs at most $K{+}NK$ pushes and $N$ pops, all inside the timer. 

\paragraph{Packed verification.}
Let $P=|x|$ be the number of cached prefix tokens. The packed input $[b,z_1,\dots,z_N]$ has position id $P$ for $b$ and $P{+}\mathrm{depth}(v_i)$ for $z_i$. RoPE positions repeat across siblings and depth 1 aligns with $q_1$. An additive mask of shape $(1,1,1{+}N,\,P{+}1{+}N)$ grants every packed token visibility of the $P$ prefix entries and of its ancestor chain inside the pack, including its own row. The ancestor matrix is built by one vectorized pass over parent pointers on the host and copied to the device once each round, because both the tree topology and the prefix length change from round to round. 

\paragraph{KV gather and feature cache.}
After the walk commits a path of length $A$, the target cache holds $P{+}1{+}N$ packed entries. The prefix plus the $A{+}1$ packed indices for the previous root and the accepted path are gathered in order for every layer. RoPE phases are already baked into stored keys at their depth-positions, which equal the tokens' final absolute positions. The DFlash feature cache is appended with the target features of exactly the committed positions, mirroring standard DFlash. 

The timer window shared by AR, DFlash, and \method is defined in Appendix~\ref{app:more}.

\paragraph{Software.}
AR, DFlash, and \method run in bf16 and share one software environment within each setting.

\paragraph{Stochastic verifier.}
A simpler direct sampled walk, which replaces $\arg\max$ by a single sample from the verifier distribution at each visited node, is also distribution-preserving but accepts one candidate for each node. Its bf16 total variation from autoregressive sampling is $0.004$, the same as that of the proposal-aware verifier.
\section{Measurement protocol and cross-hardware results}
\label{app:more}

\paragraph{Prompt sets.}
GSM8K uses the first 40 prompts of the public test split and HumanEval the first 40 problems. MT-Bench uses first turns only from the first 40 prompts of the public prompt set, which cover four of the eight categories (writing, roleplay, reasoning, and math). The H100 SXM panels use 40 prompts by domain and add MATH-500 and MBPP.

\paragraph{Timing window.}
Batch-one decoding rows use matched decode-only steady-state timing for AR, standard DFlash, and \method. For AR the timer opens after prefill, just before the first decoded token. For DFlash and \method it opens after prefill and after the first draft construction, just before the first target verification, and it closes after output truncation, with CUDA synchronization before both reads. Tokenizer work and prefill are excluded for all systems, and the first draft construction is excluded for the two speculative decoders. Every later DFlash call, target verification, tree and mask construction, packing, KV gather, and feature-cache update is included. The denominator is the emitted new-token count after the end-of-sequence token or truncation. Prompt order is fixed within a run. In Tables~\ref{tab:main}, \ref{tab:budget-sweep-gains}, and \ref{tab:same-packed}, each Blackwell value comes from the run with the median gain among three runs. The exceptions are Qwen3-8B at $N{=}127$, measured once, and Qwen3-4B at $N{=}31$, whose gain averages two runs. A6000 rows are single runs.

\paragraph{Measurement isolation.}
All headline rows are measured on otherwise idle GPUs. Co-tenant GPU load inflates the relative gain, because the memory-bound one-chain baseline degrades faster under shared-memory pressure than the compute-heavier wider-tree verify. A Blackwell-8B GSM8K run on a shared GPU measures about $+21\%$, against $+16\%$ on an idle GPU. Idle-GPU timing is therefore the conservative operating point.

\paragraph{Verify-cost probe.}
Table~\ref{tab:cost} is the component probe behind the cost law of \S\ref{sec:cost}. A transient puts the A100 and H100 NVL $n{=}32$ cells high, recovering by $n{=}48$. Blackwell-8B then jumps to $30.3$\,ms at $n{=}128$, a kernel or tile boundary that the 4B target does not hit in range. On the H100 SXM node the cost rises by $+6.7\%$ (8B) and $+7.4\%$ (4B) from $n{=}16$ to $n{=}128$ with no jump, a slope of $0.018$ and $0.017$\,ms for each added packed token, which sends the rule to the grid's end.

The 4B target is far cheaper than 8B ($c^{\mathrm{fw}}_1\approx0.015$ against $0.056$\,ms for each added packed token) with no cliff, so the rule predicts a wider tree for the smaller target on identical hardware, while the A6000 and the H100 SXM node ($0.018$\,ms for 8B) bracket the flat extreme, where forward cost alone never caps width in range.

\begin{table}[h]
\centering
\small
\caption{Target-forward latency (ms) by packed-token count $n{=}N{+}1$, using bf16 and a 1k-token KV prefix. Each probe uses random packed tokens, eight warmups, and 30 timed iterations. The H100 SXM rows use scaled dot-product attention.}
\label{tab:cost}
\begin{tabular}{lrrrrrrrr}
\toprule
\rowcolor{grouprow}
GPU, target & $n{=}1$ & 16 & 24 & 32 & 48 & 64 & 96 & 128\\
\midrule
Blackwell, Qwen3-8B & 17.2 & 21.9 & 23.3 & 23.6 & 24.1 & 24.6 & 24.4 & 30.3\\
A6000, Qwen3-8B     & 117.0 & 128.1 & 128.8 & 131.5 & 129.6 & 129.6 & 133.2 & 129.1\\
A100, Qwen3-8B      & 30.2 & 34.4 & 33.2 & 39.4 & 34.9 & 35.3 & 35.9 & 35.6\\
H100 NVL, Qwen3-8B  & 14.7 & 17.9 & 17.8 & 20.8 & 18.2 & 18.4 & 19.2 & 19.2\\
Blackwell, Qwen3-4B & 16.9 & 18.7 & 19.6 & 19.5 & 19.4 & 19.4 & 20.9 & 20.4\\
\midrule
H100 SXM, Qwen3-8B  & 24.7 & 30.6 & 30.2 & 30.6 & 31.2 & 31.5 & 32.2 & 32.6\\
H100 SXM, Qwen3-4B  & 23.9 & 25.3 & 25.1 & 25.3 & 25.7 & 25.9 & 26.6 & 27.1\\
\bottomrule
\end{tabular}
\end{table}

\paragraph{Prefix length.}
Table~\ref{tab:cost-prefix} gives the H100 SXM probe at 1k, 4k, 8k, and 16k tokens of KV prefix, measured in one session separate from Table~\ref{tab:cost}, together with the LLaMA-3.1-8B cell of Appendix~\ref{app:llama}. At the 1k prefix both sessions show a nearly flat cost, with a Qwen3-8B slope of $-0.016$\,ms for each added packed token in Table~\ref{tab:cost-prefix} and $0.018$ from Table~\ref{tab:cost}. The slope is at most $0.029$\,ms at every prefix length, and the rule stays at $N^\ast{=}127$ throughout.

\begin{table}[h]
\centering
\small
\caption{H100 SXM target-forward latency (ms) across KV-prefix lengths at batch one in bf16. Slope is the latency increase in ms for each added packed token from $n{=}16$ to $128$.}
\label{tab:cost-prefix}
\begin{tabular}{lrrrrrr}
\toprule
\rowcolor{grouprow}
Prefix & $n{=}1$ & 16 & 32 & 64 & 128 & Slope\\
\midrule
\grouphead{7}{Qwen3-8B}\\
1k & 23.6 & 26.6 & 25.1 & 25.2 & 24.8 & -0.016\\
4k & 22.7 & 25.5 & 25.4 & 26.1 & 26.6 & 0.010\\
8k & 22.9 & 35.9 & 36.2 & 37.5 & 37.9 & 0.018\\
16k & 23.7 & 56.7 & 58.0 & 59.9 & 60.0 & 0.029\\
\grouphead{7}{Qwen3-4B}\\
1k & 23.2 & 24.7 & 25.3 & 24.8 & 24.5 & -0.002\\
4k & 23.3 & 24.6 & 25.1 & 25.1 & 24.8 & 0.001\\
8k & 22.5 & 32.2 & 33.1 & 34.6 & 34.9 & 0.024\\
16k & 22.4 & 54.4 & 55.6 & 57.4 & 57.5 & 0.027\\
\grouphead{7}{LLaMA-3.1-8B}\\
1k & 16.7 & 17.8 & 18.0 & 18.4 & 17.9 & 0.001\\
\bottomrule
\end{tabular}

\end{table}

\paragraph{Cross-hardware cells.}
Table~\ref{tab:main} gives the rule-predicted cells, and Table~\ref{tab:budget-sweep-gains} gives the full gain sweep. The A6000 $N{=}127$ cells ($\ddagger$) are measured on a second A6000 host. On that host, the mean gain at $N{=}127$ is $+25.1\%$ for 8B and $+31.7\%$ for 4B. Both lie below the same host's $N{=}95$ means ($+26.6\%$ and $+32.6\%$), so the flat A6000 curve does not reward the last grid step and the predicted $N^\ast{=}95$ remains the A6000 oracle in Table~\ref{tab:rule}.

\begin{table}[h]
\centering
\footnotesize
\setlength{\tabcolsep}{2pt}
\caption{Decoding on Blackwell and A6000 for Qwen3-8B (left) and Qwen3-4B (right). Cells report decode-only ms/token, gain over DFlash at $N^\ast$, and committed round length $R$.}
\label{tab:main}
\begin{minipage}[t]{0.495\linewidth}\centering
\begin{tabular}[t]{l rrr r r}
\toprule
 & \multicolumn{3}{c}{Decode time (ms/token)} & & \\
\cmidrule(lr){2-4}
Domain & AR & DFlash & \method & Gain & $R$\\
\midrule
\grouphead{6}{Blackwell 8B, $N^\ast{=}63$}\\
GSM8K & 18.0 & 4.25 & \textbf{3.65} & \textbf{\textcolor{posdelta}{+16.4\%}} & 7.97\\
MT-Bench & 18.0 & 8.23 & \textbf{6.56} & \textbf{\textcolor{posdelta}{+25.4\%}} & 4.40\\
HumanEval & 18.0 & 4.06 & \textbf{3.45} & \textbf{\textcolor{posdelta}{+17.6\%}} & 8.51\\
\addlinespace[1pt]
\grouphead{6}{A6000 8B, $N^\ast{=}95$}\\
GSM8K & 119.8 & 26.95 & \textbf{21.74} & \textbf{\textcolor{posdelta}{+24.0\%}} & 8.17\\
MT-Bench & 122.2 & 53.38 & \textbf{39.82} & \textbf{\textcolor{posdelta}{+34.1\%}} & 4.55\\
HumanEval & 123.1 & 26.46 & \textbf{20.67} & \textbf{\textcolor{posdelta}{+28.0\%}} & 8.67\\
\bottomrule
\end{tabular}
\end{minipage}\hfill
\begin{minipage}[t]{0.495\linewidth}\centering
\begin{tabular}[t]{l rrr r r}
\toprule
 & \multicolumn{3}{c}{Decode time (ms/token)} & & \\
\cmidrule(lr){2-4}
Domain & AR & DFlash & \method & Gain & $R$\\
\midrule
\grouphead{6}{Blackwell 4B, $N^\ast{=}95$}\\
GSM8K & 19.3 & 4.59 & \textbf{3.70} & \textbf{\textcolor{posdelta}{+24.3\%}} & 8.13\\
MT-Bench & 17.9 & 8.17 & \textbf{6.10} & \textbf{\textcolor{posdelta}{+33.9\%}} & 4.66\\
HumanEval & 16.7 & 3.85 & \textbf{3.11} & \textbf{\textcolor{posdelta}{+23.6\%}} & 8.48\\
\addlinespace[1pt]
\grouphead{6}{A6000 4B, $N^\ast{=}95$}\\
GSM8K & 120.1 & 27.46 & \textbf{21.39} & \textbf{\textcolor{posdelta}{+28.4\%}} & 8.12\\
MT-Bench & 122.9 & 54.52 & \textbf{39.18} & \textbf{\textcolor{posdelta}{+39.1\%}} & 4.70\\
HumanEval & 134.8 & 27.21 & \textbf{20.86} & \textbf{\textcolor{posdelta}{+30.4\%}} & 8.58\\
\bottomrule
\end{tabular}
\end{minipage}
\end{table}

\begin{table}[h]
\centering
\scriptsize
\setlength{\tabcolsep}{4pt}
\caption{Greedy width-sweep gains over DFlash across GPUs and targets. Bold marks each row's oracle. Underline marks the predicted width. $\ddagger$ marks cells measured on a second A6000 host.}
\label{tab:budget-sweep-gains}
\begin{tabular}{lllrrrrrr}
\toprule
\rowcolor{grouprow}
GPU & target & domain & $N{=}15$ & $N{=}31$ & $N{=}47$ & $N{=}63$ & $N{=}95$ & $N{=}127$\\
\midrule
Blackwell & Qwen3-8B & GSM8K     & \textcolor{posdelta}{+5.9\%} & \textcolor{posdelta}{+13.1\%} & \textcolor{posdelta}{+15.5\%} & \underline{\textcolor{posdelta}{+16.4\%}} & \textbf{\textcolor{posdelta}{+17.8\%}} & $\textcolor{negdelta}{-0.8\%}$\\
Blackwell & Qwen3-8B & MT-Bench  & \textcolor{posdelta}{+19.1\%} & \textcolor{posdelta}{+21.8\%} & \textcolor{posdelta}{+24.8\%} & \underline{\textcolor{posdelta}{+25.4\%}} & \textbf{\textcolor{posdelta}{+26.9\%}} & \textcolor{posdelta}{+7.5\%}\\
Blackwell & Qwen3-8B & HumanEval & \textcolor{posdelta}{+8.2\%} & \textcolor{posdelta}{+12.5\%} & \textcolor{posdelta}{+17.3\%} & \underline{\textcolor{posdelta}{+17.6\%}} & \textbf{\textcolor{posdelta}{+19.7\%}} & \textcolor{posdelta}{+0.5\%}\\
\midrule
Blackwell & Qwen3-4B & GSM8K     & \textcolor{posdelta}{+12.2\%} & \textcolor{posdelta}{+17.9\%} & \textcolor{posdelta}{+22.8\%} & \textcolor{posdelta}{+22.7\%} & \underline{\textcolor{posdelta}{+24.3\%}} & \textbf{\textcolor{posdelta}{+27.1\%}}\\
Blackwell & Qwen3-4B & MT-Bench  & \textcolor{posdelta}{+29.2\%} & \textcolor{posdelta}{+31.6\%} & \textcolor{posdelta}{+34.2\%} & \textcolor{posdelta}{+34.5\%} & \underline{\textcolor{posdelta}{+33.9\%}} & \textbf{\textcolor{posdelta}{+37.9\%}}\\
Blackwell & Qwen3-4B & HumanEval & \textcolor{posdelta}{+12.3\%} & \textcolor{posdelta}{+14.8\%} & \textcolor{posdelta}{+21.8\%} & \textcolor{posdelta}{+19.9\%} & \underline{\textbf{\textcolor{posdelta}{+23.6\%}}} & \textcolor{posdelta}{+22.0\%}\\
\midrule
A6000 & Qwen3-8B & GSM8K     & \textcolor{posdelta}{+12.4\%} & \textcolor{posdelta}{+19.1\%} & \textcolor{posdelta}{+23.3\%} & \textbf{\textcolor{posdelta}{+24.1\%}} & \underline{\textcolor{posdelta}{+24.0\%}} & \textcolor{posdelta}{+22.0\%}$^\ddagger$\\
A6000 & Qwen3-8B & MT-Bench  & \textcolor{posdelta}{+22.8\%} & \textcolor{posdelta}{+31.2\%} & \textcolor{posdelta}{+33.1\%} & \textbf{\textcolor{posdelta}{+34.7\%}} & \underline{\textcolor{posdelta}{+34.1\%}} & \textcolor{posdelta}{+29.4\%}$^\ddagger$\\
A6000 & Qwen3-8B & HumanEval & \textcolor{posdelta}{+16.5\%} & \textcolor{posdelta}{+23.9\%} & \textcolor{posdelta}{+26.3\%} & \textcolor{posdelta}{+26.5\%} & \underline{\textbf{\textcolor{posdelta}{+28.0\%}}} & \textcolor{posdelta}{+23.8\%}$^\ddagger$\\
\midrule
A6000 & Qwen3-4B & GSM8K     & \textcolor{posdelta}{+12.7\%} & \textcolor{posdelta}{+22.8\%} & \textcolor{posdelta}{+25.5\%} & \textbf{\textcolor{posdelta}{+28.5\%}} & \underline{\textcolor{posdelta}{+28.4\%}} & \textcolor{posdelta}{+27.2\%}$^\ddagger$\\
A6000 & Qwen3-4B & MT-Bench  & \textcolor{posdelta}{+17.6\%} & \textcolor{posdelta}{+32.7\%} & \textcolor{posdelta}{+33.9\%} & \textbf{\textcolor{posdelta}{+39.2\%}} & \underline{\textcolor{posdelta}{+39.1\%}} & \textcolor{posdelta}{+38.7\%}$^\ddagger$\\
A6000 & Qwen3-4B & HumanEval & \textcolor{posdelta}{+15.6\%} & \textcolor{posdelta}{+6.3\%} & \textcolor{posdelta}{+28.5\%} & \textcolor{posdelta}{+28.0\%} & \underline{\textbf{\textcolor{posdelta}{+30.4\%}}} & \textcolor{posdelta}{+29.1\%}$^\ddagger$\\
\bottomrule
\end{tabular}
\end{table}

\paragraph{Same packed length.}
Table~\ref{tab:same-packed} isolates tree shape from width. The $N{=}15$ tree sends exactly 16 packed tokens, the same count standard DFlash verifies, and the only change is shape. Across all twelve cells this lifts $R$ by $+0.69$ to $+1.17$ and runs $5.9$ to $29.2\%$ faster before any wider tree is allowed. These fixed-budget cells are measured directly, without sweep selection. Width beyond that, up to the rule-predicted $N^\ast$ of Table~\ref{tab:main}, adds a mean of $9.0$ percentage points of gain, ranging from $4.7$ to $12.1$ points across the six Blackwell cells.

\begin{table}[h]
\centering
\small
\setlength{\tabcolsep}{4pt}
\caption{Tree-shape ablation at 16 packed tokens, comparing standard DFlash with \method at $N{=}15$. Cells report committed round length $R$, decode-only ms/token, and gain over DFlash on 40 prompts in each domain.}
\label{tab:same-packed}
\label{tab:width-shape}
\begin{tabular}{lllrrrrr}
\toprule
\rowcolor{grouprow}
GPU & target & domain & DFlash $R$ & $N{=}15$ $R$ & DFlash ms & $N{=}15$ ms & shape gain\\
\midrule
Blackwell & Qwen3-8B & GSM8K     & 6.02 & 6.71 & 4.40 & 4.16 & \textbf{\textcolor{posdelta}{+5.9\%}}\\
Blackwell & Qwen3-8B & MT-Bench  & 3.13 & 3.90 & 8.40 & 7.05 & \textbf{\textcolor{posdelta}{+19.1\%}}\\
Blackwell & Qwen3-8B & HumanEval & 6.36 & 7.26 & 4.11 & 3.80 & \textbf{\textcolor{posdelta}{+8.2\%}}\\
\midrule
Blackwell & Qwen3-4B & GSM8K     & 5.93 & 6.85 & 4.25 & 3.79 & \textbf{\textcolor{posdelta}{+12.2\%}}\\
Blackwell & Qwen3-4B & MT-Bench  & 3.11 & 4.00 & 8.77 & 6.79 & \textbf{\textcolor{posdelta}{+29.2\%}}\\
Blackwell & Qwen3-4B & HumanEval & 6.15 & 7.19 & 4.10 & 3.65 & \textbf{\textcolor{posdelta}{+12.3\%}}\\
\midrule
A6000 & Qwen3-8B & GSM8K     & 6.04 & 6.89 & 26.95 & 23.98 & \textbf{\textcolor{posdelta}{+12.4\%}}\\
A6000 & Qwen3-8B & MT-Bench  & 3.12 & 3.91 & 53.38 & 43.45 & \textbf{\textcolor{posdelta}{+22.8\%}}\\
A6000 & Qwen3-8B & HumanEval & 6.28 & 7.45 & 26.46 & 22.72 & \textbf{\textcolor{posdelta}{+16.5\%}}\\
\midrule
A6000 & Qwen3-4B & GSM8K     & 5.94 & 6.74 & 27.46 & 24.36 & \textbf{\textcolor{posdelta}{+12.7\%}}\\
A6000 & Qwen3-4B & MT-Bench  & 3.13 & 3.93 & 54.52 & 46.38 & \textbf{\textcolor{posdelta}{+17.6\%}}\\
A6000 & Qwen3-4B & HumanEval & 6.22 & 7.22 & 27.21 & 23.54 & \textbf{\textcolor{posdelta}{+15.6\%}}\\
\bottomrule
\end{tabular}
\end{table}

\paragraph{Sampled decoding on Blackwell.}
Table~\ref{tab:t1-sampled} gives the three-seed $T{=}1$ rows of the proposal-aware stochastic verifier on Blackwell-8B at $N^\ast{=}63$. The gain over the chain is $+16.7\%$ to $+30.1\%$ with seed ranges of a few points.

\begin{table}[h]
\centering
\small
\setlength{\tabcolsep}{3pt}
\caption{Stochastic verification for Qwen3-8B on Blackwell at $T{=}1$ and $N^\ast{=}63$. Cells report committed round length $R$, decode-only ms/token, mean gain across seeds 0--2, and the seed-level gain range.}
\label{tab:t1-sampled}
\begin{tabular}{lrrrrrr}
\toprule
\rowcolor{grouprow}
domain & DFlash $R$ & \method $R$ & DFlash ms/tok & \method ms/tok & mean gain & seed range\\
\midrule
GSM8K     & 5.70 & 7.54 & 4.88 & 4.18 & \textbf{\textcolor{posdelta}{+16.7\%}} & $+16.0$ to $\textcolor{posdelta}{+18.0\%}$\\
MT-Bench  & 3.69 & 5.10 & 9.43 & 7.25 & \textbf{\textcolor{posdelta}{+30.1\%}} & $+28.7$ to $\textcolor{posdelta}{+32.3\%}$\\
HumanEval & 5.66 & 7.72 & 5.11 & 4.18 & \textbf{\textcolor{posdelta}{+22.3\%}} & $+20.3$ to $\textcolor{posdelta}{+25.3\%}$\\
\bottomrule
\end{tabular}
\end{table}

\subsection{Hybrid linear-attention targets}
\label{app:hybrid_cost}

We repeat the component probe on two public hybrid checkpoints, Qwen3.6-27B (dense, three linear-attention layers for every full-attention layer) and Qwen3.6-35B-A3B (the same interleaving with 256-expert sparse mixture-of-experts (MoE) blocks, 8 active), on one Blackwell card with fused gated-delta-rule kernels. Linear-attention layers mutate their recurrent state in place on every forward. Each timed forward therefore receives an untimed copy of the prefix cache, and we report the median of 15 timed forwards after 5 warmups.

\begin{table}[t]
\centering
\small
\caption{Verify-forward latency (ms) across packed-token counts for Qwen3.6 hybrid targets on Blackwell, bf16, batch one. Values are medians of 15 timed forwards. Each receives an untimed copy of the prefix cache.}
\label{tab:cost-hybrid}
\begin{tabular}{llrrrrrrr}
\toprule
\rowcolor{grouprow}
Target & Prefix & $n{=}1$ & 2 & 16 & 64 & 128 & 256 & 1024 \\
\midrule
27B dense hybrid & 1k & 46.9 & 67.9 & 69.1 & 68.5 & 68.6 & 74.6 & 194.7 \\
27B dense hybrid & 8k & 47.3 & 68.5 & 70.0 & 78.2 & 80.2 & 93.9 & 237.2 \\
35B MoE hybrid & 1k & 87.3 & 109.7 & 129.8 & 158.4 & 175.4 & 194.7 & 229.3 \\
35B MoE hybrid & 8k & 88.2 & 111.1 & 133.9 & 163.7 & 179.8 & 198.4 & 246.2 \\
\bottomrule
\end{tabular}
\end{table}

\paragraph{Two regimes and their consequence.}
Table~\ref{tab:cost-hybrid} shows a two-regime curve on both targets. A single-token pass rides the fused recurrent decode kernel and is context-flat. Any multi-token pass switches to the chunked kernel and pays a fixed entry of about $+21$\,ms ($45\%$) on the 27B target and $+22$ to $23$\,ms ($26\%$) on the MoE target, after which the curve is locally affine again. On the dense hybrid at the 1k prefix the latency stays within $5\%$ across $n\in[2,128]$, with a slope of about zero over $n\in[16,128]$, against about $0.1$\,ms for each added token at 8k. At the 1k prefix a 63-node tree therefore verifies at the cost of the 16-token chain to within noise on the dense hybrid, while the MoE hybrid keeps a prefix-independent slope of $0.39$ to $0.40$\,ms for each added token and growth resumes by $n\approx256$ on the dense hybrid. Concurrent engine work already provides tree verification for gated linear-attention layers \citep{arxiv-2607-06763,arxiv-2607-16673}.

\subsection{Statistics and output equivalence}
\label{app:stats}

\paragraph{Bootstrap confidence intervals.}
To bound prompt-selection variance, the Blackwell cells are measured again at the rule-predicted widths on expanded prompt sets of 200 GSM8K, 80 MT-Bench, and 164 HumanEval prompts. Each prompt is timed three times, and the prompt means are resampled with a paired prompt-level bootstrap ($10{,}000$ resamples). Table~\ref{tab:bootstrap} reports the mean prompt-level gain over standard DFlash with its 95\% confidence interval. Every interval lies well above zero.

\paragraph{Run-to-run spread.}
Across the three runs behind each Blackwell-8B cell at $N^\ast{=}63$, the gain varies by at most $2.3$ points in every domain.

\paragraph{Greedy equivalence.}
In fp32, \method and standard DFlash both reproduce AR decoding token for token in every evaluated domain and budget on Qwen3-8B and Qwen3-4B. In bf16, over 600 prompts across budgets, both decoders depart from AR only at near-ties of the top two target logits, where shape-dependent bf16 rounding can flip the greedy token.

\paragraph{Sampling distribution.}
For the $T{=}1$ rows, the tokens committed by the proposal-aware stochastic verifier are consistent with the target sampling law. In fp64 their total variation from autoregressive sampling is $0.003$, below the $0.013$ Monte-Carlo floor. In bf16 it is $0.004$, matching the direct sampled walk of Appendix~\ref{app:impl}. Appendix~\ref{app:serving} reports the same comparison under serving load.
\section{H100 SXM panels in full}
\label{app:h100sxm}

\paragraph{Evaluation protocol.}
The H100 SXM panels use 40 prompts from each of GSM8K \citep{cobbe2021gsm8k}, MT-Bench \citep{zheng2023mtbench}, HumanEval \citep{chen2021humaneval}, MATH-500 \citep{lightman2023math500}, and MBPP \citep{austin2021mbpp}, with 256 new tokens (512 for MATH-500 and MBPP) and the decode-only timer of Appendix~\ref{app:more}.

\paragraph{Width sweep.}
Table~\ref{tab:h100sxm} gives the full ms/token sweep behind Table~\ref{tab:h100main}. Two further repetitions of the whole panel reproduce every chain-to-\method gain within $3.7$ points, and no cell shows the cliff reversal of the Blackwell-8B sweep.

\begin{table}[h]
\centering
\small
\setlength{\tabcolsep}{4.4pt}
\caption{Qwen3 greedy width sweeps on H100 SXM, using the public block-16 DFlash head of each target. Cells report mean decode-only ms/token. $R$ denotes the committed round length of DFlash and of the widest \method tree.}
\label{tab:h100sxm}
\begin{tabular}{lrrrrrrrrr}
\toprule
\rowcolor{grouprow}
Domain & max new & AR & chain & $N{=}47$ & $N{=}63$ & $N{=}95$ & $N{=}127$ & $R$ chain & $R$ at $127$\\
\midrule
\grouphead{10}{Qwen3-8B}\\
GSM8K     & 256 & 24.85 & 5.48 & 4.50 & 4.40 & 4.35 & \textbf{4.28} & 6.23 & 8.46\\
MT-Bench  & 256 & 24.79 & 10.75 & 8.04 & 7.86 & \textbf{7.83} & 7.85 & 4.06 & 5.52\\
HumanEval & 256 & 24.77 & 5.31 & 4.26 & 4.12 & 4.06 & \textbf{4.02} & 6.52 & 9.04\\
MATH-500  & 512 & 24.84 & 4.46 & 3.82 & 3.85 & 3.70 & \textbf{3.69} & 7.81 & 9.87\\
MBPP      & 512 & 25.39 & 5.27 & 4.22 & 4.14 & 4.08 & \textbf{4.06} & 6.56 & 9.09\\
\grouphead{10}{Qwen3-4B}\\
GSM8K     & 256 & 23.87 & 5.42 & 4.33 & 4.26 & 4.15 & \textbf{4.13} & 6.13 & 8.47\\
MT-Bench  & 256 & 24.72 & 10.75 & 7.87 & 7.77 & 7.62 & \textbf{7.50} & 3.95 & 5.78\\
HumanEval & 256 & 23.88 & 5.18 & 4.16 & 4.11 & 4.01 & \textbf{3.95} & 6.49 & 9.02\\
MATH-500  & 512 & 25.21 & 4.72 & 4.02 & 3.93 & 3.91 & \textbf{3.81} & 7.49 & 9.76\\
MBPP      & 512 & 24.76 & 4.99 & 4.04 & 4.03 & 3.94 & \textbf{3.90} & 6.72 & 9.11\\
\bottomrule
\end{tabular}
\end{table}

At $T{=}1$, the tree leads the chain by $+22\%$ to $+37\%$ across the ten cells of Table~\ref{tab:h100main}, and each seed's mean tree decode time is within $3\%$ of the three-seed cell mean.

\paragraph{EAGLE-3 settings.}
The EAGLE-3 rows of Table~\ref{tab:h100main} run the official EAGLE code on the same node, prompts, and decode-only timer as the other rows, with the DFlash paper's tree settings (a draft tree of 16 tokens to match the block and of 60 as in the EAGLE-3 paper, seven draft steps, top-$k$ 10). Two public Qwen3-8B heads exist. The body rows use the stronger Tengyunw head, and Table~\ref{tab:eagle3-heads} adds the AngelSlim head the DFlash paper reports against. The 4B rows use the AngelSlim Qwen3-4B head, the only public one.

\begin{table}[h]
\centering
\small
\setlength{\tabcolsep}{4pt}
\caption{EAGLE-3 head comparison for Qwen3-8B greedy decoding on H100 SXM. Cells report decode-only ms/token and committed round length $R$ for 16- and 60-token draft trees. AR anchors are in ms/token.}
\label{tab:eagle3-heads}
\begin{tabular}{llrrrrrrrrrr}
\toprule
Head & Tree & \multicolumn{2}{c}{GSM8K} & \multicolumn{2}{c}{MATH-500} & \multicolumn{2}{c}{HumanEval} & \multicolumn{2}{c}{MBPP} & \multicolumn{2}{c}{MT-Bench}\\
\cmidrule(lr){3-4}\cmidrule(lr){5-6}\cmidrule(lr){7-8}\cmidrule(lr){9-10}\cmidrule(lr){11-12}
 & & ms/tok & $R$ & ms/tok & $R$ & ms/tok & $R$ & ms/tok & $R$ & ms/tok & $R$\\
\midrule
Tengyunw & 16 & 9.65 & 4.45 & 10.47 & 4.11 & 11.13 & 3.97 & 10.72 & 4.11 & 12.00 & 3.74\\
Tengyunw & 60 & \textbf{8.07} & \textbf{5.34} & \textbf{8.91} & \textbf{4.88} & \textbf{9.21} & \textbf{4.76} & \textbf{9.06} & \textbf{4.93} & \textbf{10.28} & \textbf{4.49}\\
AngelSlim & 16 & 12.97 & 3.26 & 13.35 & 3.11 & 13.29 & 3.18 & 13.82 & 3.10 & 16.33 & 2.64\\
AngelSlim & 60 & 11.23 & 3.78 & 11.65 & 3.59 & 11.59 & 3.65 & 11.99 & 3.58 & 14.22 & 3.05\\
\midrule
AR anchor & & \multicolumn{2}{c}{24.85} & \multicolumn{2}{c}{24.84} & \multicolumn{2}{c}{24.77} & \multicolumn{2}{c}{25.39} & \multicolumn{2}{c}{24.79}\\
\bottomrule
\end{tabular}

\end{table}

\paragraph{Generation length and long prefix.}
Table~\ref{tab:len-long} reports MATH-500 at $512$ and $2048$ new tokens and GovReport summarization prompts cut to $8$k tokens. The $8$k row of Table~\ref{tab:cost-prefix} gives $N^\ast{=}127$, and on both targets the gain stays within $4$ points from $N{=}47$ to $127$, with the best width at $95$. At $2048$ new tokens the sweep covers $N{=}63$ and $N{=}127$.

\begin{table}[h]
\centering
\small
\setlength{\tabcolsep}{4pt}
\caption{Generation length and KV prefix on H100 SXM at $T{=}0$ over 40 prompts, with \method at $N^\ast{=}127$. The long-prefix rows generate 256 new tokens. Cells report speedup over AR, committed round length $R$, gain over DFlash, and the sweep oracle.}
\label{tab:len-long}
\begin{tabular}{l r r r r r}
\toprule
\rowcolor{grouprow}
Cell & AR ms/tok & DFlash $\times$AR ($R$) & \method $\times$AR ($R$) & Gain & Oracle $N$ (gain)\\
\midrule
\grouphead{6}{Generation length, Qwen3-8B, MATH-500}\\
512 new tokens & 24.8 & 5.57 (7.81) & \textbf{6.74} (9.87) & \textcolor{posdelta}{+21\%} & 127 (\textcolor{posdelta}{+21\%})\\
2048 new tokens & 24.0 & 5.82 (8.09) & \textbf{6.99} (10.20) & \textcolor{posdelta}{+20\%} & 127 (\textcolor{posdelta}{+20\%})\\
\grouphead{6}{Long prefix, GovReport prompts cut to 8k tokens}\\
Qwen3-8B & 25.1 & 1.39 (2.24) & \textbf{1.77} (3.43) & \textcolor{posdelta}{+28\%} & 95 (\textcolor{posdelta}{+30\%})\\
Qwen3-4B & 24.6 & 1.43 (2.18) & \textbf{1.83} (3.37) & \textcolor{posdelta}{+27\%} & 95 (\textcolor{posdelta}{+30\%})\\
\bottomrule
\end{tabular}

\end{table}

\section{Serving harness}
\label{app:serving}

\begin{figure}[t]
\centering
\includegraphics[width=0.5\linewidth]{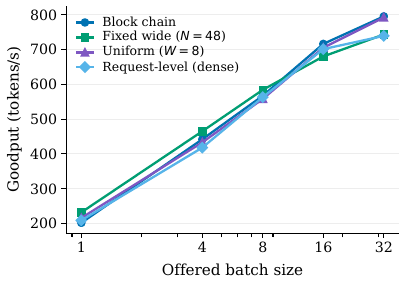}
\caption{Qwen3-8B serving goodput on Blackwell, medians of four runs.}
\label{fig:serving_goodput}
\end{figure}

The batched harness extends the single-device decoder to $B$ concurrent requests. One batched DFlash pass returns every position marginal for all requests, an allocator assigns the shared verifier budget, one batched ancestor-masked target pass verifies all packs, and a ragged left-padded target cache is compacted each round. It is a dense prototype whose verify pack is right-padded to $1+\max_r N_r$ tokens for request widths $N_r$, which charges every request the maximum allocation.

\paragraph{Greedy equivalence under batching.}
We compare every sequence committed under each allocation policy with AR greedy decoding. At $B{=}4$ over 18 prompts, every policy departs from it only at bf16 near-ties, where the top two target logits differ by at most $0.25$, and the plain DFlash baseline itself departs at $11$ such ties.

\paragraph{Block shrink against width shrink.}
Table~\ref{tab:blockshrink} separates drafter length from verifier width at matched packed budgets, with a 256-new-token cap. At 8 packed tokens, a 7-node tree drawn from the block-16 drafter is the fastest configuration in every domain, ahead of the block-8 chain and of the same-size tree from a block-8 drafter. Shrinking the tree therefore costs less than shrinking the drafter block.

\begin{table}[h]
\centering
\footnotesize
\setlength{\tabcolsep}{2.5pt}
\caption{Block-length and tree-width ablation at matched packed budgets on H100 SXM, Qwen3-8B, $T{=}0$. Cells show decode-only ms/token with committed round length in parentheses.}
\label{tab:blockshrink}
\begin{tabular}{l rrr rr}
\toprule
 & \multicolumn{3}{c}{8 packed tokens} & \multicolumn{2}{c}{16 packed tokens}\\
\cmidrule(lr){2-4}\cmidrule(lr){5-6}
Domain & chain, block 8 & tree $N{=}7$, block 8 & tree $N{=}7$, block 16 & chain, block 16 & tree $N{=}15$, block 16\\
\midrule
GSM8K & 6.84 (4.86) & 6.51 (5.19) & \textbf{6.31 (5.37)} & 5.47 (6.23) & \textbf{4.96 (7.02)}\\
MT-Bench & 11.77 (3.28) & 10.29 (3.66) & \textbf{10.07 (3.80)} & 10.82 (4.06) & \textbf{8.87 (4.68)}\\
HumanEval & 6.74 (4.96) & 6.35 (5.32) & \textbf{6.33 (5.40)} & 5.37 (6.52) & \textbf{4.76 (7.48)}\\
\bottomrule
\end{tabular}

\end{table}

\paragraph{Cost law under load.}
The drafter forward is independent of width and grows only with batch, from $3.2$\,ms at $B{=}1$ to $19.3$\,ms at $B{=}32$. In Figure~\ref{fig:serving_cost}a, the verifier-forward cost at a width of $W{=}64$ for each request over $W{=}1$ grows from $+10\%$ at $B{=}1$ to $+79\%$ at $B{=}8$ and $+237\%$ at $B{=}32$.

\paragraph{Goodput and acceptance in the dense prototype.}
Figure~\ref{fig:serving_goodput} gives goodput for four policies, the block chain, a fixed wide tree with $N{=}48$, a uniform width of $W{=}8$, and dense request-level allocation. Fixed wide leads at $B{\le}8$ and collapses at $B{\ge}16$, where the plain chain is best. Request-level allocation lifts mean accepted round length by $5.3$ to $9.9\%$ over a uniform split, but the dense pad-to-max charges that width to every request, so its goodput sits at or below uniform. The padding tax, not the allocation, is what the varlen verifier below removes.

\paragraph{Allocation oracle and the varlen verifier.}
Under ragged packing the verify cost scales with the total node count. At a fixed shared budget goodput is proportional to total accepted tokens. An offline oracle over the measured request-level acceptance (900 requests by domain, 400 random batches) shows online global best-first raising total accepted tokens over a uniform split by $+16.8\%$/$+18.2\%$/$+22.5\%$ on GSM8K/MT-Bench/HumanEval at a scarce shared budget of four nodes for each of $B{=}32$ requests, within a few points of hindsight water-filling, and by only $1$ to $2\%$ once the budget is abundant. Bucketing requests by node count does not realize this, since each extra pass re-streams the 8B weights and goodput falls by up to $45\%$, and within one dense pass the padding rows still cost $9$ to $15\%$ of the verify forward.

We therefore implement a single-pass varlen verifier as a forward-only tree-attention kernel in Triton (cumulative-sequence-length ragged layout, query-level ancestor bitmask, grouped-query attention, bf16). Its logits agree with the dense verify to fp32 tolerance, and decoding with it also departs from AR greedy only at bf16 near-ties. It lifts committed-token goodput over uniform width by a median $+8.6\%$ at $B{=}16$ and $+13.9\%$ at $B{=}32$ (four runs, ranges $[+8.1,+14.6]\%$ and $[+13.2,+19.8]\%$), where the dense request-level verify reaches $-0.5\%$ and $-6.7\%$. Over the batch-wise best static policy, which is the chain at both batch sizes, it lifts goodput by a median $+7.0\%$ at $B{=}16$ and $+13.7\%$ at $B{=}32$, with a gain in every run. Best-first allocation, which Proposition~\ref{prop:alloc} proves optimal for a shared budget, adds $9$--$10\%$ committed acceptance at $B{=}8$--$16$.

\paragraph{Continuous batching and sampling under load.}
A minimal continuous-batching loop (Poisson arrivals, dynamic admission and retirement, both policies on the same varlen store and verify so only the allocator differs) separates the allocation from the padding-tax removal. At sustained occupancy ($200$ requests) request-level best-first beats uniform width by $0.9$ to $3.6\%$ across active batches of $8$ to $32$, and the loop's outputs again depart from AR greedy only at bf16 near-ties. Under the same scheduler at $T{=}1$ each request commits through the residual recursion of Theorem~\ref{thm:stoch}, and on an H100 the served committed tokens are consistent with autoregressive sampling (fp64 total variation $0.006$, below the $0.013$ Monte-Carlo floor, its maximum over served prefixes decaying as the inverse square root of the draw count).

\section{Extended analysis}
\label{app:analysis}

The first four analyses below replay the offline Qwen3-8B greedy logs from the same 40-prompt slices as the main rows, following the budget-sweep and acceptance-length analyses of EAGLE-2 and OPT-Tree \citep{li2024eagle2,wang2025opt}.

\paragraph{Adaptive against fixed-shape trees.}
At fixed nonroot budget $N{=}47$, two fixed-shape trees do not beat the standard chain. The fixed schedule $(8,4,2,2,1,\ldots)$ reaches $R=3.7/3.0/3.7$ and uniform binary branching $4.9/3.2/4.8$, against the chain's $6.3/3.2/6.3$ on GSM8K/MT-Bench/HumanEval. The adaptive top-$N$ path-product tree reaches $7.7/4.3/7.7$.

\paragraph{Where rejections go.}
For each first rejection of standard DFlash, Figure~\ref{fig:analysis}a ranks the target greedy correction in the DFlash marginal at the rejected offset. The correction has rank 2--4 in $56$ to $77\%$ of first-rejection events (GSM8K and HumanEval $77\%$, MT-Bench $56\%$), ranks 5--8 in another $11$ to $14\%$, and lies outside the top 32 in only $2$ to $11\%$.

\paragraph{Calibration.}
A hit means that the candidate node token equals the target greedy token at that node's prefix. At depth 1, the binned reliability in Figure~\ref{fig:analysis}b is close to the diagonal. In the $[0.9,1.0]$ bin the empirical hit rate is $0.98$ to $0.99$, and mid-range bins are mildly underconfident ($[0.6,0.7]\to0.66$ to $0.77$).

\paragraph{Stopping values.}
For GSM8K on Blackwell, the threshold of the rule in \S\ref{sec:width} is about $0.015$. The sorted candidate masses fall from $0.86$ at the first candidate to $0.016$ at the 48th and $0.011$ at the 64th, so the rule predicts the optimum between $N{=}47$ and $N{=}63$, the flat region of Table~\ref{tab:budget-sweep-gains}.

\subsection{Rank-cap ablation}
\label{app:rankcap}

The public decoder of \S\ref{sec:method} retains the top $K{=}8$ tokens at each future position before the best-first expansion, while the fixed-budget problem of Theorem~\ref{thm:tree} admits the full budget-tied rank set. On the saved greedy round logs ($1914$/$3090$/$2150$ rounds on GSM8K/MT-Bench/HumanEval, Qwen3-8B target, public DFlash-b16 head) we replay both rules at matched node budget and score committed accepted length $R$ against the realized continuation. At budget $32$ the $K{=}8$ tree comes within $0.5\%$ of the full-rank tree's committed accepted length in every domain, because the realized target token sits beyond rank $8$ only $9.5$ to $15.2\%$ of the time. At the deployed width $N^\ast{=}63$ on a dedicated top-$64$ log the gap stays within $1.1\%$. The cap is near-free at the operating point while scoring $8$ candidates for each position instead of the full budget-tied rank set.

\subsection{Width transfer to DART}
\label{app:dart}

To test whether the width mechanism is specific to DFlash, we repeat the offline acceptance replay on DART \citep{arxiv-2601-19278}, an architecturally distinct one-pass drafter consisting of a single customized Transformer decoder layer that emits position-level logits for $L$ future positions in one forward. We use the public \texttt{fvliang/qwen8b-dart} head for the same Qwen3-8B target, log the top 32 tokens of its one-pass logits at each position over a greedy decode (mapped to the target vocabulary through its draft-vocabulary table), and score committed accepted length $R$ of a best-first tree with rank cap $K{=}8$ against the realized continuation. Over $1900$ replayed rounds for each benchmark, one at each position of 20 greedy decodes, $R$ grows from $1.66$ at budget $1$ to $2.81$ at budget $63$ on GSM8K ($+70\%$) and from $1.76$ to $3.15$ on HumanEval ($+79\%$). The top-1 chain over the same eight positions reaches $R=2.01$ and $2.35$, which the tree exceeds by $19.1\%$ and $14.5\%$ at the same eight nodes and by $39.9\%$ and $34.2\%$ at budget $63$. DART's $L{=}8$ caps $R$ at $9$, against $16$ for the DFlash block, but the trend is the same. We also time DART's own tree decoder, which adds its small n-gram trie and runtime tree search, against autoregressive decoding of the same target on an A6000, with 30 prompts in each domain and 256 new tokens. At node budgets of $8$ and $60$, the tree decoder is $1.91\times$ and $2.14\times$ faster than autoregressive decoding on GSM8K and $2.02\times$ and $2.34\times$ faster on HumanEval. Across budgets of $8$, $16$, $32$, and $60$, its committed round length rises at every step, from $2.30$ to $2.65$ on GSM8K and from $2.49$ to $2.92$ on HumanEval.

\section{Additional target families}
\subsection{LLaMA-3.1-8B, block 10}
\label{app:llama}

Table~\ref{tab:llama} repeats the main protocol on LLaMA-3.1-8B-Instruct with the public UltraChat-trained block-10 DFlash head and official EAGLE-3 head, using H100 SXM, bf16, batch one, 40 prompts in each of five domains, and 256 new tokens (512 for MATH-500 and MBPP). The flat cost slope of $0.001$\,ms for each added packed token selects $N^\ast{=}127$. The tree beats the chain at every width and domain, with sweeps within $2.1\%$ from $63$ to $127$. On GSM8K, EAGLE-3 at tree 60 matches the accepted length of \method, but uses eight draft passes rather than one. In decode time, \method at $N^\ast{=}127$ is $30$ to $60\%$ faster than the 60-token EAGLE-3 tree in every domain.

\begin{table}[h]
\centering
\footnotesize
\setlength{\tabcolsep}{3pt}
\caption{LLaMA-3.1-8B greedy decoding with the block-10 DFlash head on H100 SXM. Cells report speedup over AR and committed round length $R$.}
\label{tab:llama}
\begin{tabular}{lrrrrrrrrrrrr}
\toprule
Method \tabsmall\secondarycolor{(width)} & \multicolumn{2}{c}{GSM8K} & \multicolumn{2}{c}{MATH-500} & \multicolumn{2}{c}{HumanEval} & \multicolumn{2}{c}{MBPP} & \multicolumn{2}{c}{MT-Bench} & \multicolumn{2}{c}{Avg.}\\
\cmidrule(lr){2-3}\cmidrule(lr){4-5}\cmidrule(lr){6-7}\cmidrule(lr){8-9}\cmidrule(lr){10-11}\cmidrule(lr){12-13}
 & $\times$AR & $R$ & $\times$AR & $R$ & $\times$AR & $R$ & $\times$AR & $R$ & $\times$AR & $R$ & $\times$AR & $R$\\
\midrule
AR ms/token & \multicolumn{2}{c}{17.30} & \multicolumn{2}{c}{19.67} & \multicolumn{2}{c}{17.12} & \multicolumn{2}{c}{19.89} & \multicolumn{2}{c}{16.87} & \multicolumn{2}{c}{}\\
\midrule
\grouphead{13}{Autoregressive drafter}\\
EAGLE-3 \tabsmall\secondarycolor{(10)} & 2.05 & 4.60 & 1.75 & 4.05 & 2.28 & 5.15 & 2.25 & 5.11 & 1.73 & 4.14 & 2.01 & 4.61\\
EAGLE-3 \tabsmall\secondarycolor{(60)} & 2.67 & 6.08 & 2.31 & 5.28 & 2.80 & 6.56 & 2.81 & 6.40 & 2.34 & 5.59 & 2.59 & 5.98\\
\grouphead{13}{One-pass block drafter}\\
\cellcolor{baselinerow}DFlash \tabsmall\secondarycolor{(10)} & \cellcolor{baselinerow}3.00 & \cellcolor{baselinerow}4.31 & \cellcolor{baselinerow}2.76 & \cellcolor{baselinerow}4.01 & \cellcolor{baselinerow}3.60 & \cellcolor{baselinerow}5.18 & \cellcolor{baselinerow}3.56 & \cellcolor{baselinerow}5.09 & \cellcolor{baselinerow}2.32 & \cellcolor{baselinerow}3.51 & \cellcolor{baselinerow}3.05 & \cellcolor{baselinerow}4.42\\
\cellcolor{oursrow}\method \tabsmall\secondarycolor{($N{=}15$)} & \cellcolor{oursrow}3.41 & \cellcolor{oursrow}5.11 & \cellcolor{oursrow}3.23 & \cellcolor{oursrow}4.83 & \cellcolor{oursrow}4.01 & \cellcolor{oursrow}6.02 & \cellcolor{oursrow}4.00 & \cellcolor{oursrow}5.90 & \cellcolor{oursrow}2.74 & \cellcolor{oursrow}4.23 & \cellcolor{oursrow}3.48 & \cellcolor{oursrow}5.22\\
\cellcolor{oursrow}\method \tabsmall\secondarycolor{($N{=}31$)} & \cellcolor{oursrow}3.61 & \cellcolor{oursrow}5.44 & \cellcolor{oursrow}3.43 & \cellcolor{oursrow}5.08 & \cellcolor{oursrow}4.23 & \cellcolor{oursrow}6.37 & \cellcolor{oursrow}4.27 & \cellcolor{oursrow}6.30 & \cellcolor{oursrow}2.92 & \cellcolor{oursrow}4.52 & \cellcolor{oursrow}3.69 & \cellcolor{oursrow}5.54\\
\cellcolor{oursrow}\method \tabsmall\secondarycolor{($N{=}47$)} & \cellcolor{oursrow}3.75 & \cellcolor{oursrow}5.68 & \cellcolor{oursrow}3.45 & \cellcolor{oursrow}5.16 & \cellcolor{oursrow}4.32 & \cellcolor{oursrow}6.55 & \cellcolor{oursrow}4.36 & \cellcolor{oursrow}6.47 & \cellcolor{oursrow}3.02 & \cellcolor{oursrow}4.76 & \cellcolor{oursrow}3.78 & \cellcolor{oursrow}5.72\\
\cellcolor{oursrow}\method \tabsmall\secondarycolor{($N{=}63$)} & \cellcolor{oursrow}3.80 & \cellcolor{oursrow}5.85 & \cellcolor{oursrow}3.52 & \cellcolor{oursrow}5.30 & \cellcolor{oursrow}4.33 & \cellcolor{oursrow}6.65 & \cellcolor{oursrow}4.42 & \cellcolor{oursrow}6.60 & \cellcolor{oursrow}3.07 & \cellcolor{oursrow}4.86 & \cellcolor{oursrow}3.83 & \cellcolor{oursrow}5.85\\
\cellcolor{oursrow}\method \tabsmall\secondarycolor{($N{=}95$)} & \cellcolor{oursrow}3.84 & \cellcolor{oursrow}6.02 & \cellcolor{oursrow}3.55 & \cellcolor{oursrow}5.45 & \cellcolor{oursrow}4.38 & \cellcolor{oursrow}6.83 & \cellcolor{oursrow}4.48 & \cellcolor{oursrow}6.80 & \cellcolor{oursrow}3.05 & \cellcolor{oursrow}4.94 & \cellcolor{oursrow}3.86 & \cellcolor{oursrow}6.01\\
\ours\method \tabsmall\secondarycolor{($N^\ast{=}127$)} & 3.82 & 6.08 & 3.59 & 5.60 & 4.38 & 6.94 & 4.49 & 6.90 & 3.05 & 5.00 & 3.87 & 6.11\\
\bottomrule
\end{tabular}

\end{table}

\subsection{Qwen3-Coder-30B-A3B, block 16}
\label{app:coder}

This appendix repeats the main protocol on Qwen3-Coder-30B-A3B-Instruct, a 30B sparse mixture-of-experts target with 3B active parameters, using the public block-16 DFlash head for that model on the same H100 SXM node, with greedy decoding, 40 prompts each on GSM8K, HumanEval, and MBPP, and 256 new tokens. AR, DFlash, and \method all run the experts without a fused kernel.

At a 1k prefix, target-forward latency is $58.7$\,ms at $n{=}16$ and $57.4$\,ms at $128$, a flat slope within run noise. In Table~\ref{tab:coder}, the rule selects $N^\ast{=}127$, also the mean sweep oracle, with a $35.6\%$ mean gain over DFlash. Gains rise through $127$ on GSM8K and HumanEval, while MBPP peaks at $95$ and gives back $1.0$ point of gain at the final step. Mean gain improves only $0.7$ points over that step, consistent with a broad plateau rather than a cost cliff.

\begin{table}[h]
\centering
\footnotesize
\setlength{\tabcolsep}{3pt}
\caption{Qwen3-Coder-30B-A3B greedy decoding with the block-16 DFlash head on H100 SXM. Cells report speedup over AR and committed round length $R$.}
\label{tab:coder}
\begin{tabular}{lrrrrrrrr}
\toprule
Method \tabsmall\secondarycolor{(width)} & \multicolumn{2}{c}{GSM8K} & \multicolumn{2}{c}{HumanEval} & \multicolumn{2}{c}{MBPP} & \multicolumn{2}{c}{Avg.}\\
\cmidrule(lr){2-3}\cmidrule(lr){4-5}\cmidrule(lr){6-7}\cmidrule(lr){8-9}
 & $\times$AR & $R$ & $\times$AR & $R$ & $\times$AR & $R$ & $\times$AR & $R$\\
\midrule
AR ms/token & \multicolumn{2}{c}{62.3} & \multicolumn{2}{c}{63.8} & \multicolumn{2}{c}{62.6} & \multicolumn{2}{c}{}\\
\midrule
\cellcolor{baselinerow}DFlash \tabsmall\secondarycolor{(16)} & \cellcolor{baselinerow}3.84 & \cellcolor{baselinerow}4.81 & \cellcolor{baselinerow}5.86 & \cellcolor{baselinerow}7.77 & \cellcolor{baselinerow}5.08 & \cellcolor{baselinerow}6.64 & \cellcolor{baselinerow}4.93 & \cellcolor{baselinerow}6.40\\
\cellcolor{oursrow}\method \tabsmall\secondarycolor{($N{=}15$)} & \cellcolor{oursrow}4.52 & \cellcolor{oursrow}5.65 & \cellcolor{oursrow}6.53 & \cellcolor{oursrow}8.61 & \cellcolor{oursrow}5.87 & \cellcolor{oursrow}7.52 & \cellcolor{oursrow}5.64 & \cellcolor{oursrow}7.26\\
\cellcolor{oursrow}\method \tabsmall\secondarycolor{($N{=}31$)} & \cellcolor{oursrow}4.78 & \cellcolor{oursrow}6.05 & \cellcolor{oursrow}7.13 & \cellcolor{oursrow}9.40 & \cellcolor{oursrow}6.39 & \cellcolor{oursrow}8.19 & \cellcolor{oursrow}6.10 & \cellcolor{oursrow}7.88\\
\cellcolor{oursrow}\method \tabsmall\secondarycolor{($N{=}47$)} & \cellcolor{oursrow}5.05 & \cellcolor{oursrow}6.37 & \cellcolor{oursrow}7.29 & \cellcolor{oursrow}9.58 & \cellcolor{oursrow}6.50 & \cellcolor{oursrow}8.40 & \cellcolor{oursrow}6.28 & \cellcolor{oursrow}8.12\\
\cellcolor{oursrow}\method \tabsmall\secondarycolor{($N{=}63$)} & \cellcolor{oursrow}5.20 & \cellcolor{oursrow}6.58 & \cellcolor{oursrow}7.56 & \cellcolor{oursrow}9.99 & \cellcolor{oursrow}6.55 & \cellcolor{oursrow}8.53 & \cellcolor{oursrow}6.43 & \cellcolor{oursrow}8.37\\
\cellcolor{oursrow}\method \tabsmall\secondarycolor{($N{=}95$)} & \cellcolor{oursrow}5.26 & \cellcolor{oursrow}6.69 & \cellcolor{oursrow}7.72 & \cellcolor{oursrow}10.26 & \cellcolor{oursrow}\textbf{6.91} & \cellcolor{oursrow}\textbf{8.98} & \cellcolor{oursrow}6.63 & \cellcolor{oursrow}8.64\\
\ours\method \tabsmall\secondarycolor{($N^\ast{=}127$)} & \textbf{5.34} & \textbf{6.86} & \textbf{7.78} & \textbf{10.34} & 6.86 & 8.89 & \textbf{6.66} & \textbf{8.70}\\
\bottomrule
\end{tabular}

\end{table}

\section{Extended related work}
\label{app:related}

\paragraph{One-pass draft trees.}
DART also builds trees over one-pass parallel logits. Its drafter, described in Appendix~\ref{app:dart}, is a single Transformer layer, and it prunes candidates with an external n-gram continuity signal and a trie at runtime \citep{arxiv-2601-19278}, answering which candidates to keep under a supplied construction rule or node budget. \method asks how many nonroot nodes the loop should pack on the measured hardware.

JetSpec trains a causal-parallel draft head under a tree-causal mask, where every tree node is predicted in one forward pass yet conditions only on its ancestor path, and pairs it with best-first expansion and a paged tree-attention kernel \citep{arxiv-2606-18394}. Weaver keeps the pretrained DFlash drafter and trains a lightweight autoregressive adapter over its top-$K$ marginals to restore the conditional dependencies a factorized drafter drops, together with rollback-free tree verification for gated linear-attention layers \citep{arxiv-2607-06763}. Both target the acceptance ceiling of branch-agnostic marginals.

\paragraph{Tree verification systems.}
SpecInfer \citep{miao2024specinfer} builds token trees from small speculative models and verifies valid prefixes with topology-aware tree attention, Medusa \citep{cai2024medusa} adds self-drafting heads with Cartesian or calibration-pruned sparse trees, and Sequoia \citep{chen2024sequoia} studies hardware-aware speculative trees for serving.

\paragraph{Adaptive tree and budget objectives.}
EAGLE-2 \citep{li2024eagle2} expands the current frontier, scores nodes by pathwise draft confidence, and reranks shallow and deep nodes into a connected tree. OPT-Tree \citep{wang2025opt} maximizes a draft-probability proxy for expected acceptance and stops draft deepening when the gain falls below a threshold tied to draft-step cost. Both optimize before the target pass for AR or EAGLE-style drafters.

\paragraph{Verification rules and traversal.}
Block verification \citep{sun2024block} jointly verifies a sampled draft block, HSD \citep{zhou2026hsd} adds hierarchical branch resampling, Traversal Verification \citep{weng2025traversal} accepts leaf-to-root sequences, and SpecTr \citep{sun2023spectr} and Towards Optimal Multi-draft SD \citep{arxiv-2502-18779} formulate distribution-preserving multi-candidate acceptance through optimal transport.  For the greedy walk, acceptance certificates show that a tree with branching factor $m$ certifies acceptance at a target margin about $\sqrt{(m+1)/2}$ times smaller than single-token greedy under the same divergence budget \citep{arxiv-2606-30265}.

\paragraph{Adaptive controllers.}
At the serving level, MagicDec maps where speculation pays \citep{arxiv-2408-11049}, TETRIS allocates a batch's draft-token budget \citep{arxiv-2502-15197}, FASER manages draft and verify phases under dynamic load \citep{arxiv-2604-20503}, a utility gate bounds worst-case slowdown for MoE targets \citep{arxiv-2506-20675}, and MoESD identifies when speculation accelerates sparse MoE targets \citep{arxiv-2505-19645}.

Three concurrent controllers adapt scalar draft budgets for the same one-pass block drafter. D-cut keeps the top survival-ranked prefixes under a profiled latency table \citep{arxiv-2607-14647}, AdaFlash distills on policy and truncates draft length with a learned head \citep{arxiv-2607-19223}, and BlockPilot classifies a sample-adaptive block size \citep{arxiv-2606-31315}. All adapt the depth or granularity of one chain. 

\paragraph{Objectives and theory.}
Acceptance-aligned objectives \citep{zhou2024distillspec,goel2024tvdpp,hass2025} change the trained drafter, DistillSpec by aligning draft and target distributions and HASS by targeting the training and decoding mismatch. \citet{yin2024theoretical} analyze limits of unbiased rejection-based speculative decoding. Surveys of speculative decoding round out the context \citep{arxiv-2401-07851,arxiv-2502-19732}.

\end{document}